\documentclass[10pt,final,journal]{IEEEtran}
\ifCLASSINFOpdf
\else
\fi
\usepackage{graphicx}
\usepackage{subfigure}
\usepackage{mathrsfs}
\usepackage{amsfonts}
\usepackage{multirow}
\usepackage{amsthm,amsmath}
\usepackage{bm}
\usepackage[unicode=true,
 bookmarks=true,bookmarksnumbered=true,bookmarksopen=true,bookmarksopenlevel=1,
 breaklinks=false,pdfborder={0 0 0},pdfborderstyle={},backref=false,colorlinks=false]
 {hyperref}
\newcommand{\tabincell}[2]{\begin{tabular}{@{}#1@{}}#2\end{tabular}}
\newtheorem{lemma}{Lemma}
\newtheorem{theorem}{Theorem}

\begin{document}
%
\title{A Manifold-based Airfoil Geometric-feature Extraction and Discrepant Data Fusion Learning Method}
%
%
%

\author{
Yu~Xiang \href{https://orcid.org/0000-0001-9622-7661}{\includegraphics[scale=0.5]{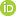}},~\IEEEmembership{Member,~IEEE,}
Guangbo~Zhang \href{https://orcid.org/0000-0002-3365-7071}{\includegraphics[scale=0.5]{ORCIDiD.png}},
Liwei~Hu \href{https://orcid.org/0000-0003-4994-9252}{\includegraphics[scale=0.5]{ORCIDiD.png}},~\IEEEmembership{Student Member,IEEE,}
Jun~Zhang \href{https://orcid.org/0000-0002-3516-0392}{\includegraphics[scale=0.5]{ORCIDiD.png}},~\IEEEmembership{Student Member,IEEE,}
and~Wenyong~Wang* \href{https://orcid.org/0000-0003-4095-547X}{\includegraphics[scale=0.5]{ORCIDiD.png}}, ~\IEEEmembership{Member,~IEEE}

\thanks{This work was supported by the basic scientific research project of central universities A030202063008039.}
\thanks{Y. Xiang, associate professor, is with School of Computer Science and Engineering, University of Electronic Science and Technology of China, Chengdu 611731, China, e-mail: jcxiang@uestc.edu.cn}
\thanks{G. Zhang, is with School of Computer Science and Engineering, University of Electronic Science and Technology of China, Chengdu 611731, China, e-mail: guangbozhang@uestc.edu.cn}
\thanks{L. Hu, is with School of Computer Science and Engineering, University of Electronic Science and Technology of China, Chengdu 611731, China, e-mail: liweihu@std.uestc.edu.cn.}
\thanks{J. Zhang, senior engineer, is with School of Computer Science and Engineering, University of Electronic Science and Technology of China, Chengdu 611731, China, e-mail: zhangjun@uestc.edu.cn.}
\thanks{W. Wang, professor, corresponding author, is with School of Computer Science and Engineering, University of Electronic Science and Technology of China, Chengdu 611731, China, e-mail: wangwy@uestc.edu.cn.}
\thanks{W. Wang, special-term professor, is with International Institute of Next Generation Internet, Macau University of Science and Technology, Macau 519020, China.}
\thanks{Manuscript received April 19, 2005; revised August 26, 2015.}
\thanks{NOTICE1:the reference of the TAES paper need to be updated after published.}
}

%
%

\markboth{Journal of \LaTeX\ Class Files,~Vol.~14, No.~8, August~2015}%
{Shell \MakeLowercase{\textit{et al.}}: Bare Demo of IEEEtran.cls for IEEE Journals}
%



\maketitle

\begin{abstract}
Geometrical shape of airfoils, together with the corresponding flight conditions, are crucial factors for aerodynamic performances prediction. The obtained airfoils geometrical features in most existing approaches (e.g., geometrical parameters extraction, polynomial description and deep learning) are in Euclidean space. State-of-the-art studies showed that curves or surfaces of an airfoil formed a manifold in Riemannian space. Therefore, the features extracted by existing methods are not sufficient to reflect the geometric-features of airfoils. Meanwhile, flight conditions and geometric features are greatly discrepant with different types, the relevant knowledge of the influence of these two factors that on final aerodynamic performances predictions must be evaluated and learned to improve prediction accuracy. Motivated by the advantages of manifold theory and multi-task learning, we propose a manifold-based airfoil geometric-feature extraction and discrepant data fusion learning method (MDF) to extract geometric-features of airfoils in Riemannian space (we call them manifold-features) and further fuse the manifold-features with flight conditions to predict aerodynamic performances. Experimental results show that our method could extract geometric-features of airfoils more accurately compared with existing methods, that the average MSE of re-built airfoils is reduced by 56.33\%, and while keeping the same predicted accuracy level of $C_{L}$, the MSE of $C_{D}$ predicted by MDF is further reduced by 35.37\%.
\end{abstract}


%
\IEEEpeerreviewmaketitle

\section{Introduction}
%
%
%
%
\IEEEPARstart{G}{eometrical} shape of airfoil greatly affects the aerodynamic performances \cite{hu2022flow,wang2021airfoil,hu2020neural}. Most commonly used airfoil geometry description method is to define a set of airfoil geometry parameters, such as chord length, maximum thickness, leading edge radius, etc. Although, these parameters are effective to perceive the variations in airfoil geometry structures, the designer of airfoils need to manipulate them manually to obtain a smooth and continuous curve, which causes limited applications \cite{shelton1993optimization}. Polynomials are other alternative efficient mathematical approaches to approximate airfoil curves, (e.g., Bézier curve \cite{derksen2010bezier} and B-spline \cite{farin2014curves}, etc) which usually employ linear combinations of high-degree polynomials to approximately describe the geometry variations of airfoil structures. Nevertheless, polynomial approaches can only give approximate expressions of airfoils and they are not powerful enough to deeply exploit features of airfoils from different space. Therefore, it is difficult for polynomials to describe some modern complex airfoils comprehensively \cite{zhu2014intuitive}.

Manifold theory has been applied to the optimization of modern complex airfoil geometry structures with manifold mapping method \cite{nagawkar2021single, li2022machine}. Du et al. applied manifold mapping to align the high-fidelity mode and the low-fidelity model to obtain the optimal target based on the performance distribution (i.e., Mach number and pressure coefficient) in inverse design \cite{du2019aerodynamic}. Raul et al. used manifold mapping within a surrogate-based optimization framework utilized for aerodynamic shape optimization to alleviate airfoil dynamic stall \cite{raul2021multifidelity}. The above studies focused on the optimization problems of airfoils, and to the best of our knowledge, there are few studies on the airfoils geometric-features extraction using manifold theory.

In addition to the geometrical shape of airfoils, flight conditions are other factors affecting aerodynamic performances. Once airfoil geometrical shape is fixed, the aerodynamic performances will be highly influenced under different flight conditions, and it will become more complicated with variable geometrical shapes under different flight conditions. The influence of these two different types of data on aerodynamic performances is called the discrepancy of aerodynamic data. As a result, it is hard to quantificationally evaluate the discrepancy of geometrical shape and flight conditions on aerodynamic performances, which also leads to the difficulty of fusing them. Recently, some researchers have recognized the importance of aerodynamic data discrepancy \cite{dakua2013performance}, and tried to create new models that learn discrepant aerodynamic data in a distributed way according to the type of input data, for examples, AeroCNN-I \cite{zhang2018application}, physics guided machine learning (PGML) \cite{pawar2021physics}, multi-layer perceptron (MLP) \cite{xin2022surrogate} and Multi-task learning (MTL) scheme \cite{hu2022aerodynamic, zhang2021amulti}. All the above studies, however, take the coordinates or images of airfoils as representations of airfoil structures in Euclidean space.

In recent years, deep learning has achieved great success in feature extraction of modern complex airfoils and discrepant aerodynamic data fusion. For examples, with radial basis function-based generative adversarial networks (RBF-GANs)\cite{hu2022flow}, convolution neural networks (CNNs)\cite{zhang2018application} and Auto-Encoders \cite{yonekura2019framework}, the feature maps of airfoils can be extracted. With MTL \cite{hu2022aerodynamic}, the discrepant aerodynamic data can be fused effectively. Nevertheless, the problems of these approaches are: 1) they rarely use manifold theory to extract latent geometric-features from Riemannian Space; and 2) the latent geometric-features of airfoils are not considered in the fusion method of discrepant aerodynamic data.

Motivated by MTL and manifold theory, we propose a manifold-based airfoil geometric-feature extraction and discrepant data fusion learning method (MDF) to extract latent geometric-features in Riemannian space (we call them manifold-features) and further fuse the manifold-features of airfoils with flight conditions to predict aerodynamic performances, see Fig.\ref{fig_structure}. Our proposed MDF consists of three modules: manifold-based airfoil geometric-feature extraction module, fight conditions input module and multi-tusk learning module. In manifold-based airfoil geometric-feature extraction module, a set of self-intersection-free Bézier curves are employed to build an smooth segmented manifold from airfoil coordinates. Then Riemannian metric of the manifold we built is calculated as a sort of manifold-feature. The smooth segmented manifold and the extracted Riemannian metric together form a smooth Riemannian manifold. In fight conditions input module, the flight conditions of airfoils are normalized. The MTL module, a discrepant data (i.e., geometric-features and flight conditions) fusion learning method, is applied to fuse the extracted geometric-features of airfoils and flight conditions to further predict aerodynamic performances \cite{hu2022aerodynamic}. The output of MTL is the predicted aerodynamic performance parameters (e.g., lift coefficient $C_{L}$ and drag coefficient $C_{D}$) of airfoils, which are used to evaluate whether the geometric-features and the flight conditions can be fused to predict aerodynamic performances precisely.

To summarize, the contributions to our work are:
\begin{enumerate}
\itemsep=0pt
\item we prove that a set of self-intersection-free Bézier curves which are connected end to end forms a smooth segmented Riemannian manifold that could describe the geometric shape of airfoils;
\item we propose a manifold-based airfoil geometric-feature extraction method using manifold metric calculated with Riemannian manifold constructed above;
\item we propose a novel discrepant data fusion learning method MDF, to fuse the Riemannian manifold features of foils in Riemannian space and the flight conditions together to predict aerodynamic performances that superior to existing methods.
\end{enumerate}

The structure of the remainder of this paper is as follows. Section.\ref{section_related_work} introduces the research status of airfoil feature extraction and discrepant data fusion in the field of airfoil-related modeling. In Section.\ref{section_methodology}, the prove of smooth segmented manifold and the details of MDF are elaborated. In Section. \ref{section_experimental_result}, a public UIUC airfoil dataset is applied to validate the effectiveness of MDF and the feasibility of geometric-features in terms of predicting $C_{L}$ and $C_{D}$ of airfoils. The conclusions of our work are shown in Section. \ref{section_conclusion}.

\section{Related Works}
\label{section_related_work}

In this section, we introduce current research status on feature extractions of airfoils and discrepant aerodynamic data fusion.

\subsection{Feature Extractions of Airfoils}
The Bézier curve, B-spline and NURBS are typical polynomials that were used to derive equations of airfoil curves \cite{derksen2010bezier,farin2014curves,zhang2021kriging}, which can be regarded as effective approaches for extracting features of airfoils. Among them, the Bézier curve is the most basic and common expression. An airfoil usually consists of multiple control points $\{x,y\}$ to a polynomial function. Usually, a $n$-degree Bézier curve that connects $n+1$ control points is chosen as basis to form a smooth curve which are used to approximate a part of a airfoil function. Then the airfoil curve function can be described as a liner combination of the basis. These polynomial expressions mentioned above are flexible and they can be combined with other parameterization methods to describe airfoil characteristics more accurately \cite{farin2014curves}.

The class function/shape function transformation (CST) \cite{kulfan2009modification}, developed based on polynomials, is the mainstream method for extracting features in the field of airfoil parameterization. The CST uses both the class function and the shape function to control the airfoil shape. The class function is applied to generate the basic shape of airfoils, and the shape function is used to correct the basic shape so as to obtain an accurate airfoil shape. The coefficients of the class function and shape function are parameters to be determined in CST.

These polynomial approaches that are mathematically interpretable were widely used to parameterize airfoils. However, there have been some frontier studies that show that curves or surfaces of an airfoil exist in manifold space \cite{du2019aerodynamic, nagawkar2021single}. Hence, existing polynomial approaches can only capture geometric-features from Euclidean space, and some latent geometric-features (e.g., manifold-features) are omitted. On the contrary, manifold theory can extract geometric-features from the perspective of manifold space and further to enrich airfoil features.

\begin{figure*}[h!]
    \centering
    \includegraphics[scale=0.9]{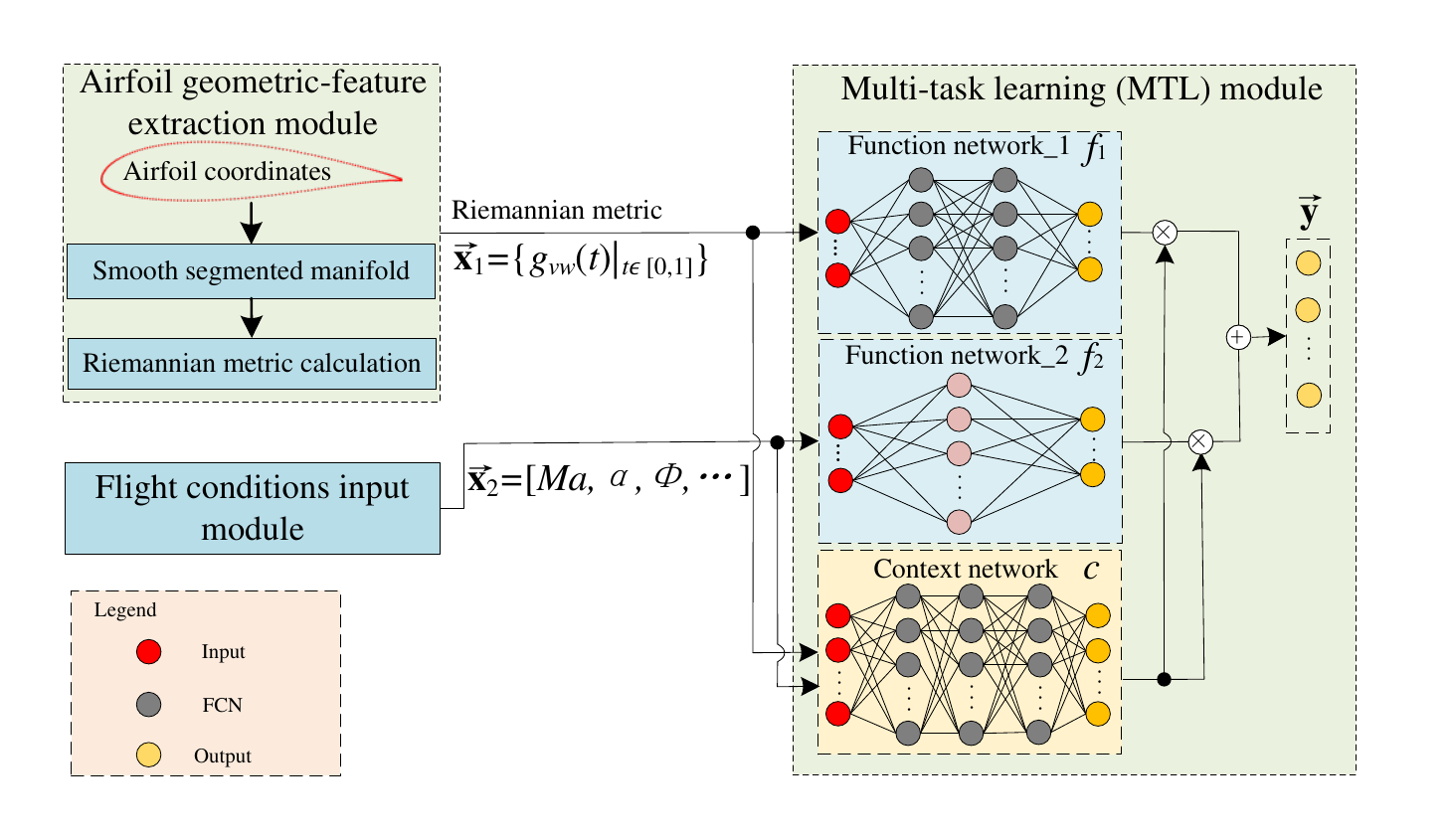}
    \caption{The structure of MDF.}
    \label{fig_structure}
\end{figure*}

\subsection{Discrepant Aerodynamic Data Fusion}

The MTL scheme is an effective approach to fuse discrepant aerodynamic data \cite{white2020fast,hu2022aerodynamic}. For general neural networks (except neural networks with dropout), all neurons will be activated for each input sample, which is difficult to adapt to the discrepancy of aerodynamic data. On the contrary, MTL is a partially activated neural network that activates different neurons according to different inputs.

The origin of MTL can be traced back to mixtures of experts (i.e., dedicated neural networks) in natural language processing \cite{jacobs1991adaptive}. The idea of mixtures of experts is that only one specific expert network is activated for analyzing a word according to the part-of-speech tags of the input word. The final results are obtained by linearly weighting the results of all experts \cite{shazeer2017outrageously}. The experts can be implemented by different models, such as SVMs \cite{rajaei2013human}, gaussian processes \cite{deisenroth2015distributed}, neural networks \cite{zhang2019learning}, etc. After 2019, the mechanism of partial activation of neurons is developed in the form of MTL in the field of aerodynamics.

The MTL is a novel neural network in that different tasks are assigned to different subnetworks. White et al. proposed the ClusterNet (a precursor to MTL) in 2020, which is used to predict the velocity of air flows \cite{white2020fast}. Zhang et al. adopted a clusterNet-based physics-informed model to predict future trajectories of the swarm by approximating the nonlinear dynamics of the swarm model \cite{zhang2020learning}. Based on ClusterNet, Hu and Zhang, et al. divided the aerodynamic data into different subtasks according to the discrepancy of them and proposed the MTL method to further fuse them to predict aerodynamic performances \cite{hu2022aerodynamic, zhang2021amulti}.

By adopting the state-of-the-art MTL scheme, the above studies focused on the discrepancy of aerodynamic data. However, they take the coordinates or images of airfoils as representations of airfoil structures in Euclidean space, and did not consider to extract latent geometric-features of airfoils from Riemannian Space.

As a conclusion, our proposed MDF is different from both the existing feature extraction of airfoils and discrepant data fusion. The output of existing feature extractions are polynomials with certain coefficients, which are used to represent structure of airfoils. Relatively, a polynomial with certain coefficients is just an intermediate of MDF. The final output of MDF are the predicted performance parameters. Besides, the input of existing discrepant data fusion methods are the discrepant data from Euclidean space. However, MDF fuse the manifold-features extracted from Riemannian space and the flight conditions from Euclidean space together to predicted aerodynamic performances of airfoils.

\section{Methodology}
\label{section_methodology}

\subsection{Overview of MDF}
Fig.\ref{fig_structure} describes the proposed structure of MDF. MDF is mainly compose of three modules: manifold-based airfoil geometric-feature extraction module, fight conditions input module and MTL module. In manifold-based airfoil geometric-feature extraction module, first, a set of self-intersection-free Bézier curves that connected end to end are chosen to form a smooth segmented manifold within airfoil shape space. Then, manifold metric calculation module is responsible for calculating the Riemannian metric on the manifold which forms a Riemannian manifold. Riemannian metric, calculated by the inner-product of arbitrary two vectors from tangent space of the Riemannian manifold, forms the base of other manifold-features of Riemannian manifold, through which the length, volume, connection and other manifold-features of Riemannian manifold can be calculated \cite{delso2022gravitomagnetic}. Therefore, Riemannian metric could be used as a basic representation of manifold-features from the tangent space of our smooth segmented Riemannian manifold within airfoil shape space. In fight conditions input module, the flight conditions of airfoils are normalized. In MTL module, the calculated Riemannian metric $\vec{\mathbf{x}}_{1}$ of airfoils is taken as one input to the function network\_1 and the corresponding flight conditions $\vec{\mathbf{x}}_{2}$ are taken as another input to the function network\_2. The inputs of the context network in MTL module are the vectorized combinations of Riemannian metric of airfoils and the flight conditions (i.e., $[\vec{\mathbf{x}}_{1},\vec{\mathbf{x}}_{2}]$). The output of the MTL module are the predicted aerodynamic performances $\vec{\mathbf{y}}$. In this section, we introduce the details of manifold-based airfoil geometric-feature extraction module and the MTL module.

\subsection{Manifold-based Airfoil Geometric-feature Extraction Module}
\subsubsection{Construction of Smooth Segmented Manifold of Airfoils}

In this section, first, we prove that a self-intersection-free Bézier curve forms a smooth manifold.

\begin{theorem}
\label{theorm_manifold}
A self-intersection-free Bézier curve forms a smooth manifold.
\end{theorem}

\begin{proof}
\label{proof_of_maanifold}
Given an 2D airfoil coordinate set $D=\{P_{i}=(x_{i},y_{i})|i=1,2,...,M\}$, where $M$ denotes the number of coordinate points. A Bézier curve can be built:

\begin{equation}
\label{equ_bezier}
r(D;t)=\sum_{i=0}^{n} P_{i} B_{i, n}(t), t \in[0,1]
\end{equation}
where $\mathrm{r}(D;t)$ is the function of Bézier curve, $D$ denotes the sample space where the airfoil coordinates located, $t$ is the parameter of this Bézier curve, $n$ is the degree of Bézier curve and $B_{i, n}(t)$ denotes the coefficient which satisfies:

\begin{equation}
\small
\label{equ_bezier_coefficient}
B_{i, n}(t)=C_{n}^{i} t^{i}(1-t)^{n-i}=\frac{n !}{i !(n-i) !} t^{i}(1-t)^{n-i}[i=0,1, \cdots, n]
\end{equation}

A $n$-degree Bézier curve with $n \ge 3$ may have self-intersections (Fig.\ref{fig_bezier_intersection} (a)), which cannot be used to construct a manifold \cite{said2021generalized}. There are two solutions to avoid the self-intersections, see Fig.\ref{fig_bezier_intersection} (b) and (c). In the subgraph (b), the self-intersection $A$ is deleted and the remaining curves construct a manifold. In the subgraph (c), the sequence (or position) of four control points makes the Bézier curve has no self-intersections. A Bézier curve without self-intersections is the prerequisite of Theorem. \ref{theorm_manifold}.

\begin{figure}[h!]
    \centering
    \subfigure[]{
       \includegraphics[scale=0.22]{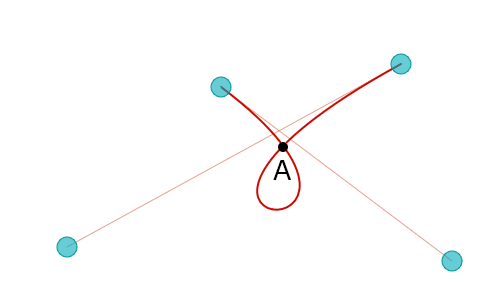}
    }
    \quad
    \subfigure[]{
       \includegraphics[scale=0.22]{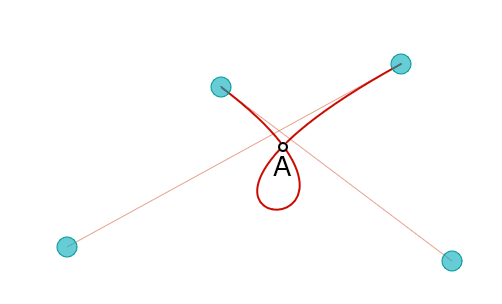}
    }
    \quad
    \subfigure[]{
       \includegraphics[scale=0.22]{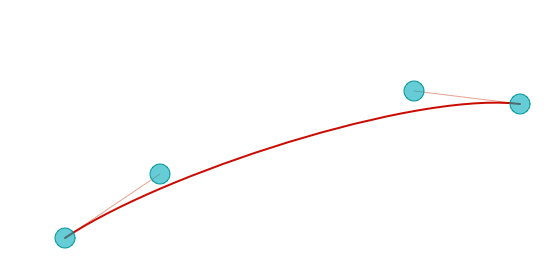}
    }
    \caption{A 3-degree Bézier curve with a self-intersection (a), a 3-degree Bézier curve remored the self-intersection (b) and a 3-degree Bézier curve without self-intersections (c). }
    \label{fig_bezier_intersection}
\end{figure}

We introduce the following lemma \cite{RiemannianTang, RiemannianChen}:
\begin{lemma}\label{lemma_bezier_manifold}
Let $\mathscr{M'}$ be a non-empty Hausdorff space. $\mathscr{M'}$ is called an m-dimensional topological manifold if, for every point $p \in \mathscr{M'}$, there exists an open neighborhood $U'$ of the point $p$ that satisfied $U' \subset \mathscr{M'}$  and a homeomorphic $\phi : U' \to \mathscr{R}^{m}$ from $U'$ to an open set of the $m$ dimensional Euclidean space $\mathscr{R}^{m}$.
\end{lemma}

We assume that $\mathscr{M'^{\star}}$ is the space where a Bézier curve $r(D;t)$ located, and $\mathscr{M'^{\star}} \subset \mathscr{R}$, then ${\exists} \mathcal{T}$, s.t. $(r,\mathcal{T})$ constitutes a topological space (i.e., $\mathscr{M'^{\star}}$ is a non-empty Hausdorff space). This is the prerequisite of Lemma \ref{lemma_bezier_manifold}. To describe the proof easily, we let $\mathscr{M'^{\star}}=\mathscr{M'}$.

Let the Euclidean space where $D=\{P_{i}=(x_{i},y_{i})|i=1,2,...,M\}$ located be $\mathscr{R}^{2}$. Because of the smoothness of $r(D;t)$, for ${\forall} t_{0} \in r(D;t)$, there must exists an open neighborhood $U'$ s.t. $t_{0} \in U'$. For ${\forall}t_{1}, t_{2} \in U'$, and $t_{1} \neq t_{2}$, ${\exists} P_{1} \in \mathscr{R}^{2}$ s.t. $r:t_{1} \to P_{1}$ and ${\exists} P_{2} \in \mathscr{R}^{2}$ s.t. $r:t_{2} \to P_{2}$.

Suppose that $P_{1} = P_{2}$, then $P_{1} = r(D;t_{1})$ and $P_{2} = r(D;t_{2})$, therefore $r(D;t_{1})=r(D;t_{2})$ which contradicts the functional properties of $r(D;t)$ (i.e., $P_{1} \neq P_{2}$).  $r(D;t)$ is a homeomorphic from $U'$ to $\mathscr{R}^{2}$.

According to Lemma. \ref{lemma_bezier_manifold}, there exists a homeomorphic $r(D;t): U' \to \mathscr{R}^{2}$, s.t. the set of all elements in $\mathscr{M'}$ is a 2-dimensional topology manifold.

The proof is finished.
\end{proof}

To approximate simple airfoils, it is sufficient to use simple curve with one Bézier function. However, in order to approximate a complex airfoil, as the number of control points of  Bézier curve increases, the degree of Bézier curve also increases, which could results in a larger error \cite{zhu2014intuitive}. It is necessary to separate the airfoil curve into multiple segments and use a set of Bézier curves instead to approximate a complex airfoil. In this paper, multiple Bézier curve segments are end to end connected, each Bézier curve segment is determined by four control points, see Fig. \ref{fig_segmented_manifold}. In this figure, the control points A, B, C, and D determine segment 1, and the control points D, E, F and G determine segment 2, and so on. Multiple smooth segments are connected end to end to form a segmented smooth curve of airfoils.

\begin{figure}[h!]
    \centering
    \includegraphics[scale=0.4]{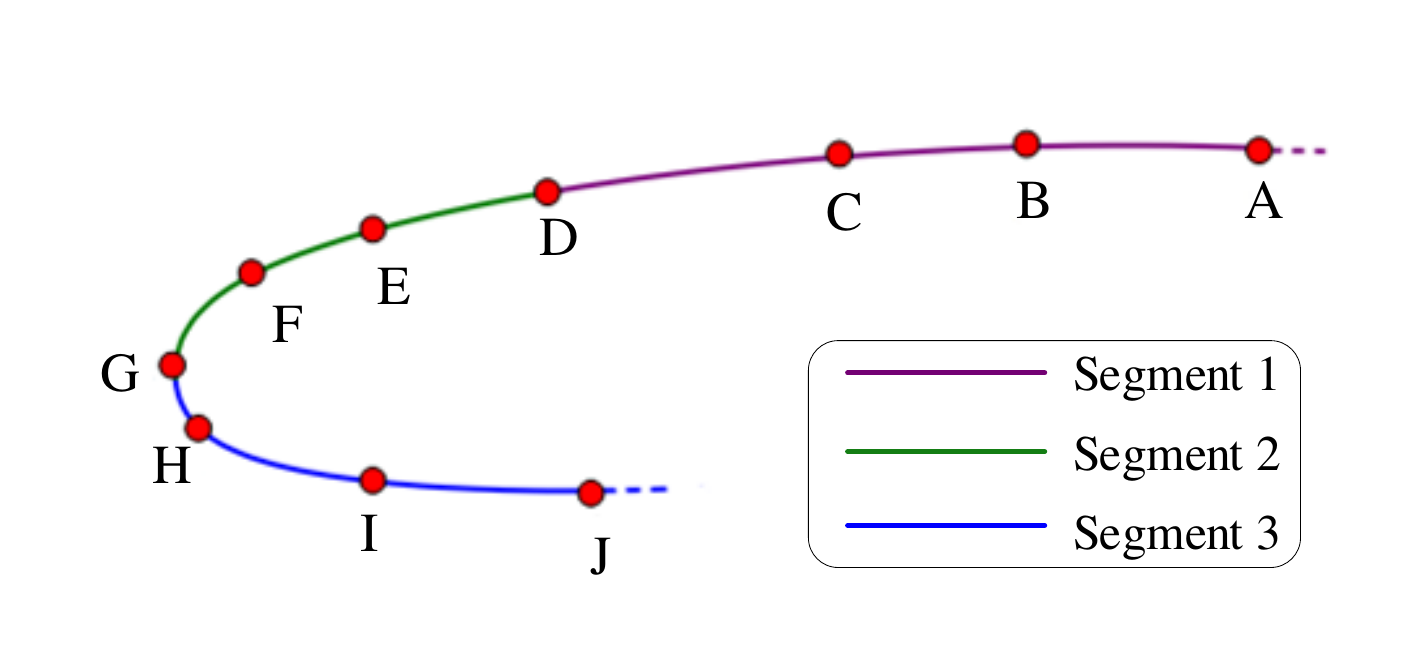}
    \caption{The construction of topology manifold based on 3-degree Bézier curve.}
    \label{fig_segmented_manifold}
\end{figure}

Next, we prove that a set of Bézier curves are connected end to end to form a smooth segmented topology manifold.

\begin{theorem}
\label{theorm_proof_of_segmented_manifold}
Multiple smooth topology manifold constructed by a set of Bézier curves are connected end to end to form a smooth segmented manifold.
\end{theorem}

\begin{proof}
To prove Theorem. \ref{theorm_proof_of_segmented_manifold}, we begin with arbitrary smooth function $r_{i}(D_{i};t)$ defined on $\mathscr{M'}_{i}$, where $i=1,2, \cdots, N$, and $N$ denotes the number of segments. We need to prove that $r_{i}(D_{i};t)$ can be extend to the whole manifold $\mathscr{M}$ which satisfies $\mathscr{M}= \cup_{i=1}^{N} \mathscr{M'}_{i}$. Therefore, Theorem. \ref{theorm_proof_of_segmented_manifold} can be rewritten as the following form:

Let $U$ be an open set in a smooth manifold $\mathscr{M}$, and $r(D;t) \in C^{1}(U)$, for $\forall P_{i} \in U$, there must exists a neighborhood $W$ that satisfies $P_{i} \in W \subset \mathscr{M'} \subset U$ and function $ \tilde{r}_{i}(D;t) \in C^{1}(W)$, s.t. $\tilde{r}_{i}(D;t)|_{W} = r(D;t)|_{W}$.

To prove $\tilde{r}(D;t)|_{W} = r(D;t)|_{W}$, we introduce a lemma of cut-off function \cite{RiemannianTang, RiemannianChen}:
\begin{lemma}
\label{lemma_decomposition}
Let $V', V$ be two open subset in a smooth topology manifold $\mathscr{M}$, and $\bar{V'}$ is a compact subset that satisfied $\bar{V'} \subset V$. Then $ {\exists} f \in C^{\infty}(M)$, s.t. $0 \leq f \leq 1$ and $f|_{V'} \equiv 1, f|_{M \setminus V} \equiv 0$.
\end{lemma}

Lemma.\ref{lemma_decomposition} means that arbitrary function defined on $\mathscr{M}$ can be smoothly truncated by multiplying with $f$:

\begin{equation}
\label{equ_cutoff}
\begin{cases}
  \left.(f \cdot r(D ; t))\right|_{V'} \equiv r(D ; t)_{V'}\\
  \left.(f \cdot r(D ; t))\right|_{\mathscr{M'} \setminus V} \equiv 0
  \end{cases}
\end{equation}

For arbitrary $P_{i} \in U$, take two open neighborhood $W$ and $V$, s.t. $\bar{W}$ and $\bar{V}$ are compact and $\bar{W} \subset V \subset \bar{V} \subset U$. According Lemma. \ref{lemma_decomposition} and (\ref{equ_cutoff}), $ {\exists} f \in C^{1}(\mathscr{M})$, s.t. $f|_{W} \equiv 1, f|_{\mathscr{M} \setminus V} \equiv 0$.

For ${\forall} P_{j}=r(D;t_{j}) \in \mathscr{M}$, we let

\begin{equation}
\label{equ_r}
\tilde{r}_{i}(D_{i};t_{j})=
\begin{cases}
r(D;t_{j})f(P_{j}), & P_{j} \in U \\
0, & P_{j} \notin U
\end{cases}
\end{equation}

Because $r(\cdot)f(\cdot)$ is smooth in subset $U$ and $r(D;t_{j})f(P_{j}) \equiv 0$ in $U \cap (\mathscr{M} \setminus \bar{V})$. According to (\ref{equ_r}), the function $\tilde{r}_{i}(D_{i};t_{j})$ is smooth in subset $U$ and satisfies $ \tilde{r}_{i}(D_{i};t_{j}) \equiv 0$ in $U \cap (\mathscr{M} \setminus \bar{V})$. Because $\mathscr{M} = U \cup (\mathscr{M} \setminus \bar{V})$, then $\tilde{r}_{i}(D_{i};t)$ is a function defined in $\mathscr{M}$.

Theorem. \ref{theorm_manifold} means that a Bézier curve forms a sooth manifold. Theorem. \ref{theorm_proof_of_segmented_manifold} explain that a set of Bézier curves are connected end to end to form a smooth segmented manifold. Combining these two theorems together, we learn that it is feasible to separate a modern complex airfoil curve into segments and a set of low-degree Bézier curves can be applied to approximate the segments.

The proof is finished.

\end{proof}

\subsubsection{Calculation of Riemannian Metric}
Those manifolds from which a set of Riemannian metrics can be calculated are called Riemannian manifolds \cite{RiemannianTang, sommer2020horizontal, hansen2021diffusion}. The construction of a Riemannian manifold build a bridge between 2D airfoil Euclidean space $ \mathscr{R}^{2}=\{x_{i} \in \mathscr{R},y_{i} \in \mathscr{R} \}$ and Riemannian space $\mathscr{M}$ with parameter $t$. The Riemannian metric at arbitrary point $t$ can be calculated by \cite{alexakis2020determining}:
\begin{equation}
\label{equ_geometric_feature}
g_{vw}(t)=\partial_{v} r(D;t) \partial_{w} r(D;t)
\end{equation}
where $g_{vw}(t)$ is the Riemannian metric of an airfoil at point $t$, $\partial_{v}=\frac{\partial}{\partial t^{v}}$  denotes directions of partial derivative. Considering that $r(D;t)$ is an 1D manifold, therefore, $v=w$:

\begin{equation}
\label{equ_geometric_feature_i_equals_j}
g_{vw}(t) = g_{vv}(t) =\left(\partial_{v} r(D;t)\right)^{2} = g_{ww}(t)
\end{equation}

We see from (\ref{equ_geometric_feature}) that Riemannian metric $g_{vw}(t)$ is the result of the inner-product of two arbitrary vectors in the tangent space $\{\frac{\partial}{\partial t^{v}}\}$ of $r(D;t)$, the two vectors that make up this inner-product can be extended the entire tangent space of the built Riemannian manifold. In other words, the Riemannian metric represents a sort of geometric-base from tangent space of airfoils, through which many manifold-features of the Riemannian manifold can be further measured and calculated. As a result, $g_{vw}(t)$ could be chosen as a manifold-feature that represents the geometrical characteristic of the airfoil curves.

\subsection{Multi-Task Learning (MTL) Module}

In this section, we introduce the MTL module from two aspects: the structure and the training method.

\subsubsection{The Sturcture}

As shown in Fig.\ref{fig_structure}, the MTL module consists of two function networks and a context network. One function network learns one of the discrepant data (i.e., the function network\_1 learns the Riemannian metric $\vec{\mathbf{x}}_{1}$ and the function network\_2 learns the flight conditions $\vec{\mathbf{x}}_{2}$). And the context network learns the strategy (i.e., the fusion weights) to fuse these two discrepant data (i.e., $[\vec{\mathbf{x}}_{1},\vec{\mathbf{x}}_{2}]$).

Riemannian metric $\vec{\mathbf{x}}_{1}$ and flight conditions $\vec{\mathbf{x}}_{2}$ can be written as:
\begin{equation}\label{equal_x1_x2}
  \begin{cases}
  \vec{\mathbf{x}}_{1}=g_{vw}(t)|_{t\in [0,1]}\\
  \vec{\mathbf{x}}_{2}=[Ma, \alpha, \Phi, \cdots]
  \end{cases}
\end{equation}
where, $Ma$ denotes the incoming Mach number, $\alpha$ denotes the angle of attack and $\Phi$ denotes the roll Angle of flow.

The outputs of the MTL are the aerodynamic performances $\vec{\mathbf{y}}$:
\begin{equation}
\label{equal_y}
\small
\begin{aligned}
\vec{\mathbf{y}}=&\{y_{i}|i=1,2,\cdots,K\}\\
 =&\sum_{m=1}^{K} f_{1 m}(\vec{\mathbf{x}}_{1}) * c_{m}([\vec{\mathbf{x}}_{1},\vec{\mathbf{x}}_{2}])+\sum_{n=1}^{K} f_{2 n}(\vec{\mathbf{x}}_{2}) * c_{K+n}([\vec{\mathbf{x}}_{1},\vec{\mathbf{x}}_{2}])
\end{aligned}
\end{equation}
where $\vec{\mathbf{y}}$ denotes the output vector, $f_{1m}(\vec{\mathbf{x}}_{1})$ denotes the $m$th component of output vector from function network\_1, $f_{2 n}(\vec{\mathbf{x}}_{2})$ denotes the $n$th component of output vector from function network\_2, $K$ denotes the number of output nodes in function networks, and $c_{m}([\vec{\mathbf{x}}_{1},\vec{\mathbf{x}}_{2}])$ denotes the $m$th component of output vector from context network, and $[\vec{\mathbf{x}}_{1},\vec{\mathbf{x}}_{2}]$ denotes the concatenation of $\vec{\mathbf{x}}_{1}$ and $\vec{\mathbf{x}}_{2}$.

\begin{table*}[h!]
\caption{The numbers of inputs in UIUC dataset after pre-processing.}
\label{tab_format_dataset}
\centering
\begin{tabular}{ccccc}
\hline
data type& geometrical parameters  & manifold-features  & coordinates of airfoils & flight conditions \\ \hline
number of variables& 7 & 271                   & $281 \times 2$          & 2                           \\ \hline
\end{tabular}
\end{table*}

\subsubsection{The Training Method}
In the MTL, the function network and the context network are trained alternately. The details of training method are as follows.

Step. 1: forward propagation. The loss function of MTL is:

\begin{equation}
\label{equ_function_loss}
E=\frac{1}{N} \sum_{z=1}^{N}\left(\frac{1}{K} \sum_{i=1}^{K}\left(y_{z i}-\hat{y}_{z i}\right)^{2}\right)
\end{equation}
where $N$ denotes the number of data in the dataset, $y_{zi}$ denotes the predicted value, and $\hat{y}_{zi}$ denotes the real value in the training set.

Step. 2: function networks updates.
The parameters of two function networks are updated by:
\begin{equation}
\label{equation_update_function_networks}
  \boldsymbol{\theta}_{f}=\boldsymbol{\theta}_{f}+\eta \frac{\partial E}{\partial \boldsymbol{\theta}_{f}}
\end{equation}
where, $\eta$ denotes the learning rate of MTL, $\boldsymbol{\theta}_{f}$ denotes the parameters of the function networks.

Step.3: context network updates.
The parameters of the context network are updated by:
\begin{equation}
\label{equation_update_context_networks}
  \boldsymbol{\theta}_{c}=\boldsymbol{\theta}_{c}+\eta \frac{\partial E}{\partial \boldsymbol{\theta}_{c}}
\end{equation}
where, $\boldsymbol{\theta}_{c}$ denotes the parameters of the context network.

Step. 4: repeat from Step.1 to Step.3, until the MTL converges.

\section{Experimental Results and Ayalysis}
\label{section_experimental_result}
To validate the feasibility of our proposed method, two categories of experiments based on UIUC \footnote{http:\/\/m-selig.ae.illinois.edu\/ads\/coord\_database.html} airfoil dataset were conducted. Experiments I was designed to compare the results of airfoils geometric-feature extractions with various methods. Experiments II was proposed to compare the prediction errors of aerodynamic performances when discrepant data are fused by different methods.

\begin{figure}[h!]
    \centering
    \includegraphics[scale=0.25]{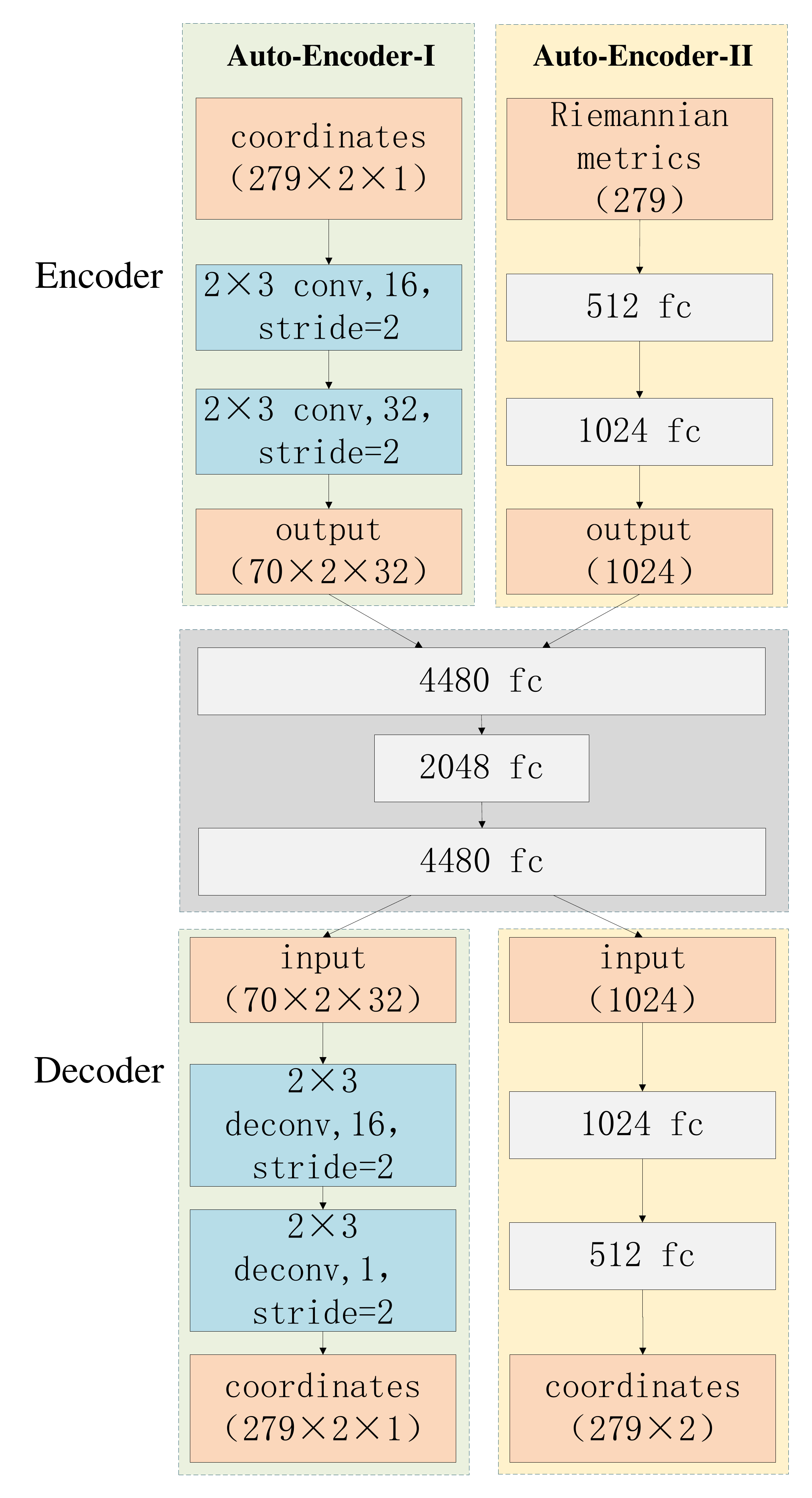}
    \caption{The structure of Auto-Encoder-I and Auto-Encoder-II. In Auto-Encoder-I, the first convolutional layer consists of sixteen 2 × 3 filters, the stride of these filters are 2 and the padding of convolution layer is `SAME'. The remaining convolution layers are similar.}
    \label{fig_autoencoder_structure}
\end{figure}

\subsection{Dataset and Pre-Processing}

UIUC dataset provides more than 1500 real airfoils, each of which is discretized by 2D coordinates. The pre-processing of UIUC airfoils data were as follows: Firstly, we removed airfoils with only upper or lower surface to guarantee that the remaining airfoils in the dataset are complete. Secondly, we sorted the remaining airfoil coordinates in the order of the trailing edge, the upper surface, the leading edge, and the lower surface to avoid the generation of self-intersections. Finally, we smoothed the airfoil coordinates through a set of 3-degree Bézier curves, and then resampled 281 coordinates for each airfoil uniformly.

Since the geometrical parameters, coordinates and manifold-features of airfoils as various input are compared in experiments, we describe the calculation of them in this paragraph. For the calculation of geometrical parameters of airfoils, we chose seven different variables as the geometrical parameters, they were chord, maximum camber, the position of maximum camber, maximum thickness, the position of maximum thickness, leading edge radius and trailing edge thickness. We used Profili to calculate the above geometrical parameters \cite{li2021analysis}. For the manifold-features of airfoils, we calculated the Riemannian metrics of airfoils at different point $t$ according to (\ref{equ_geometric_feature}). The flight conditions (i.e., $Ma$ and $\alpha$) with performances (i.e., $C_{L}$ and $C_{D}$) of the corresponding airfoils could be found on the webfoil platform \cite{du2020b}. In summary, the numbers of inputs in the UIUC dataset after pre-processing is summarized in Tab. \ref{tab_format_dataset}.

\subsection{Experiments I: Geometric-Features Comparation Experiments}

In the field of feature extractions of airfoils, the Auto-Encoders are commonly used to re-built airfoils because of their ability to extract features from geometrical shape of airfoils and represent the airfoils as latent feature vectors \cite{yonekura2022generating, yonekura2021data}. In addition, the Auto-Encoders reflect the accuracy of extracted features by calculating the similarities between re-built airfoils and real airfoils \cite{ li2022physically, morimoto2021convolutional}. Therefore, the Auto-Encoders are compared to re-built approximate airfoils with coordinates and manifold-features taking as inputs, respectively. We designed two Auto-Encoders (see Fig. \ref{fig_autoencoder_structure}). Auto-Encoder-I \cite{morimoto2021convolutional} takes airfoil coordinates as inputs and Auto-Encoder-II takes Riemannian metrics as inputs. The outputs of both the Auto-Encoder-I and Auto-Encoder-II are airfoils coordinates.

\begin{figure}[h!]
    \centering
    \subfigure[]{
        \includegraphics[scale=0.32]{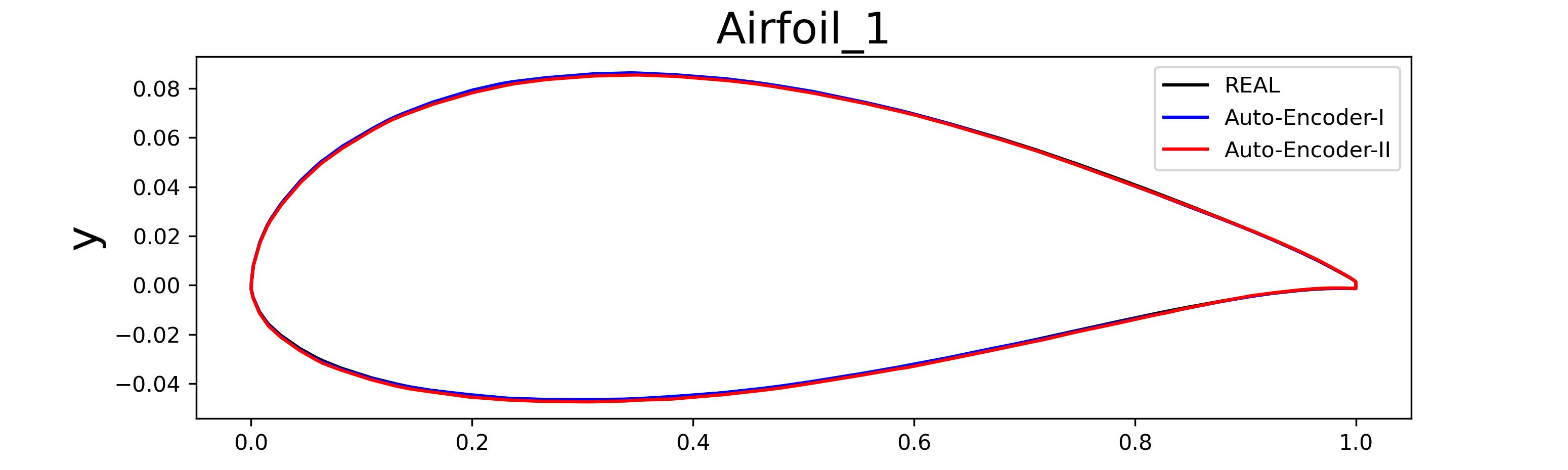}
    }
    \quad
    \subfigure[]{
        \includegraphics[scale=0.32]{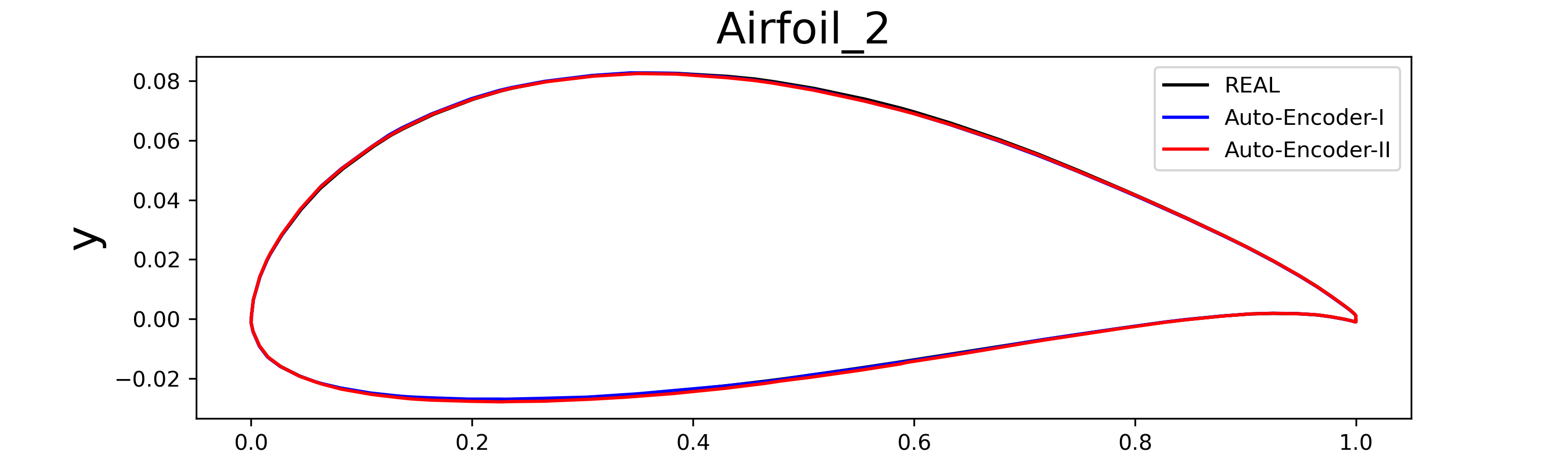}
    }
    \\
    \subfigure[]{
        \includegraphics[scale=0.32]{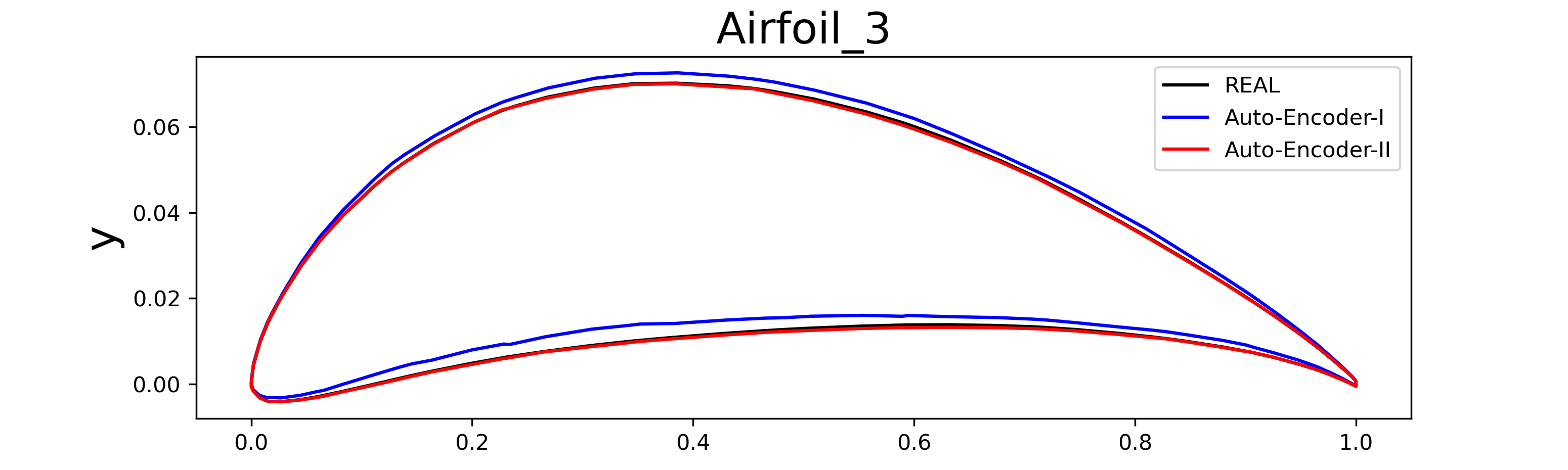}
    }
    \quad
    \subfigure[]{
        \includegraphics[scale=0.32]{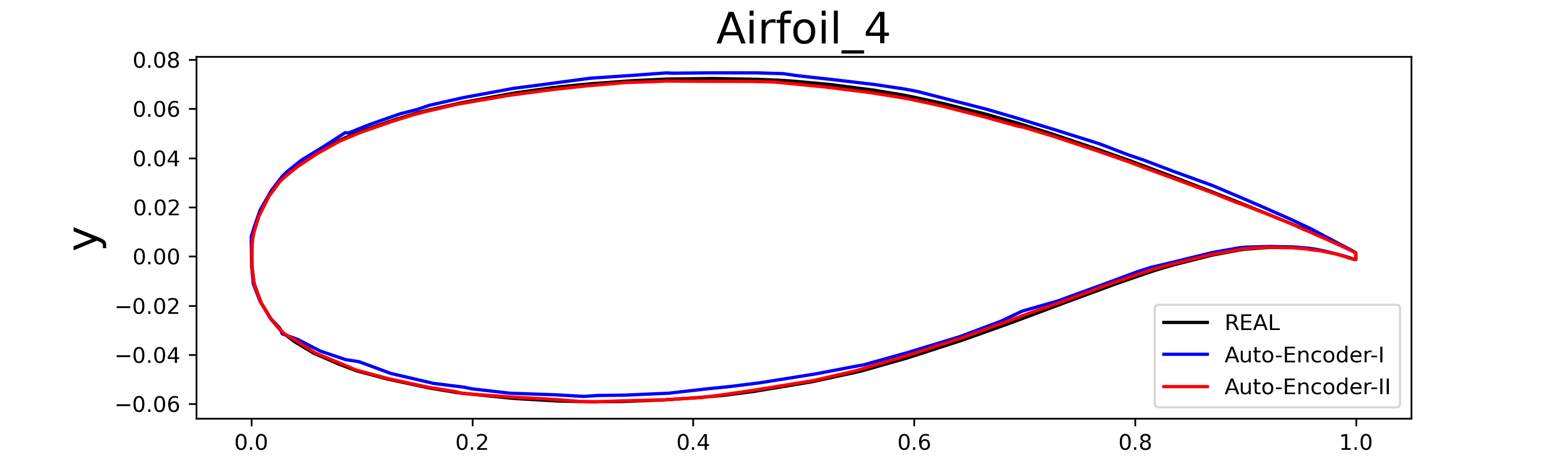}
    }
    \\
    \subfigure[]{
        \includegraphics[scale=0.32]{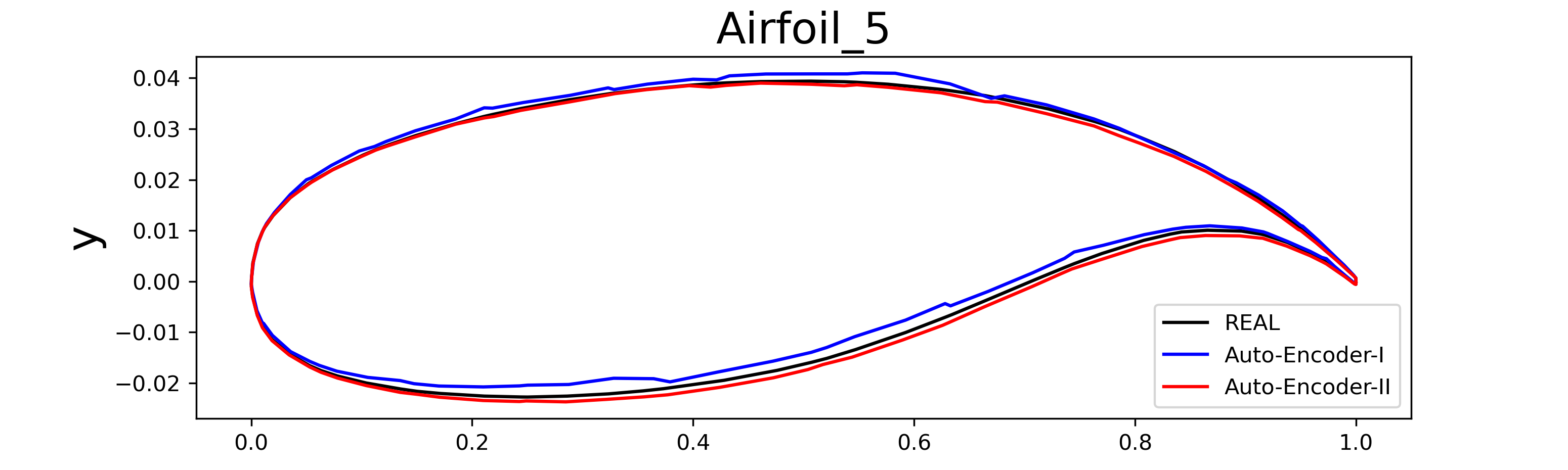}
    }
    \caption{The airfoils re-built by the Auto-Encoder-I \cite{morimoto2021convolutional} and Auto-Encoder-II.}
    \label{fig_airfoil_results}
\end{figure}

\begin{table*}[h!]
    \centering
    \caption{Hyper-parameters of deep learning models in this paper.}
    \label{tab_model_parameter}
    \begin{tabular}{ccc}
        \hline
            parameter type                              & parameter name                      &value                    \tabularnewline
        \hline
            hardware platform                           & GPU                                 &four Tesla K80 GPUs           \tabularnewline
        \hline
            \multirow{2}{*}{software platform}          & programming language                &Python 3.9.0            \tabularnewline
                                                        & deep learning framework             &pytorch 1.11.0         \tabularnewline
        \hline
            \multirow{9}{*}{hyper parameters of models}&optimization algorithm                &Adam(Beta1=0.9,Beta2=0.999) \tabularnewline
                                                        &Data normalization                   &max-min normalization    \tabularnewline
                                                        &active function                      &Relu \tabularnewline
                                                        &batch normalization                  &Yes                      \tabularnewline
                                                        &batch size                           &128                      \tabularnewline
                                                        &learning rate                        &0.001                   \tabularnewline
                                                        &epoch                                &2000                     \tabularnewline
                                                        &validation                           &10-fold cross validation  \tabularnewline
                                                        &traing set: validation set: test set &8: 1: 1  \tabularnewline
        \hline
    \end{tabular}

\end{table*}

\subsubsection{Model Settings}
The hyper-parameters in this experiments are shown in Tab.\ref{tab_model_parameter}. All the deep learning models compared in this paper are implemented based on the Pytorch framework, and they were trained 2000 epochs on four Tesla K80 GPUs \cite{zhou2022mmrotate}. In order to conduct 10-fold cross validation \cite{barnhart2021blown}, the whole dataset is randomly divided into 10 subsets. Of the 10 subsets, one single subset is selected for the testing set, and the remaining 9 subsets are selected for the training process. Among the 9 subsets, we randomly select 1 subset for the validation set, and the remaining 8 subsets are used for the training set, i.e. the training set, validation set and testing set are randomly divided at a ratio of 8:1:1. There is no intersection among all subsets, besides, every subset has the same probabilities to be chosen as a testing set. The statistical errors in this paper are the average errors of 10 folds.

The mean square error (MSE) was chosen as the loss function of the above Auto-Encoders. The loss function of Auto-Encoder-I is:
\begin{equation}
\label{equ_mse_auto_encoder_i}
\mathscr{L}_{AE-I}=\frac{1}{N} \sum_{z=1}^{N} \left( AE_{1} \left(\mathbf{P}_{z}\right)-\mathbf{P}_{z}\right)^{2}
\end{equation}
where $AE_{1}(\cdot)$ denotes the Auto-Encoder-I, $\mathbf{P}_{z} = \cup_{i=1}^M P_{i}$, a matrix with shape $M \times 2$, denotes the $z$th airfoil coordinates.

The loss function of Auto-Encoder-II is:

\begin{equation}
\label{equ_loss_function}
\mathscr{L}_{AE-II}=\frac{1}{N} \sum_{z=1}^{N} \left( AE_{2}\left(\vec{\mathrm{g}_{z}}\right)-\mathbf{P}_{z}\right)^{2}
\end{equation}
where $AE_{2}(\cdot)$ denotes the Auto-Encoder-II, and $\vec{\mathrm{g}_{z}} = \{\vec{g}^{z}_{vw}(t)|_{t \in [0,1]}\}$ denotes the vector of manifold metrics of the $z$th airfoil.

\subsubsection{Results and Analyses}

\begin{table}[h!]
\centering
\caption{The test MSE of typical five airfoils re-built by Auto-Encoder-I and Auto-Encoder-II.}
\label{tab_exp_1_errors}
  \begin{tabular}{cccc}
    \hline
    airfoils      & Auto-Encoder-I \cite{morimoto2021convolutional} & Auto-Encoder-II                & $\eta$\\ \hline
    airfoil\_1    & $5.97 \times 10^{-4}$                           & $\mathbf{3.90 \times 10^{-4}}$ &34.67\%\\
    airfoil\_2    & $6.56 \times 10^{-4}$                           & $\mathbf{2.22 \times 10^{-4}}$ &66.16\%\\
    airfoil\_3    & $1.04 \times 10^{-3}$                           & $\mathbf{1.78 \times 10^{-4}}$ &82.88\%\\
    airfoil\_4    & $1.82 \times 10^{-3}$                           & $\mathbf{1.14 \times 10^{-3}}$ &37.36\%\\
    airfoil\_5    & $1.34 \times 10^{-3}$                           & $\mathbf{7.07 \times 10^{-4}}$ &47.24\%\\
    average       & ——                                              & ——                             &53.66\% \\ \hline
  \end{tabular}
\end{table}

\begin{table*}[h!]
\caption{The details of methods compared in Experiment-II.}
\label{tab_structure_MTL}
\centering
\begin{tabular}{cccccc}
\hline
method &$\vec{x}_{1}$               & $\vec{x}_2$       & structure of network\_1 ($f_{1}$)& structure of network\_2 ($f_{2}$) & structure of $c$ \\ \hline

MTL\_g&geometrical parameters (1*7)& \tabincell{c}{flight conditions\\ ($Ma$ and $\alpha$)}&FCN: 32{*}3&FCN: 8{*}3&FCN: 32{*}3\\

MTL\_c &coordinates (281*2)& \tabincell{c}{flight conditions\\ ($Ma$ and $\alpha$)}& \tabincell{c}{conv\_1:kernel=2{*}2{*}16, stride=2\\
conv\_2:kernel=2{*}2{*}32, stride=2\\
conv\_3:kernel=2{*}2{*}64, stride=2\\
max pooling: kernel=2{*}2{*}64, stride=2} & FCN: 16{*}3& FCN: 512{*}3\\

RBF-GAN &coordinates (281*2)& \tabincell{c}{flight conditions\\ ($Ma$ and $\alpha$)}&generator FCN: 1024{*}3  &discriminator RBF: 512{*}1&——\\

PGML&coordinates (281*2)& \tabincell{c}{flight conditions\\ ($Ma$ and $\alpha$)}& FCN: 1024{*}3&FCN: 1024{*}1&——\\

MLP&coordinates (281*2)& \tabincell{c}{flight conditions\\ ($Ma$ and $\alpha$)}&FCN: 1024{*}3&——&——\\

MDF (our method) &manifold-features (1*271)& \tabincell{c}{flight conditions\\ ($Ma$ and $\alpha$)} & FCN: 1024{*}3        & FCN: 16{*}3          & FCN: 512{*}3     \\\hline
\end{tabular}
\end{table*}

The test MSEs of two Auto-Encoders on five typical airfoils are shown in Tab. \ref{tab_exp_1_errors}. The MSE reduction $\eta$ of Auto-Encoder-II is calculated by
\begin{equation}
\nonumber
\label{equ_eta}
\eta = \frac{\|\sigma_{I}-\sigma_{II}\|}{\|\sigma_{I}\|} \times 100 \%
\end{equation}
where $\sigma_{I}$ denotes the test MSE of Auto-Encoder-I, and $\sigma_{II}$ denotes the test MSE of Auto-Encoder-II. From Tab. \ref{tab_exp_1_errors}, we see that the MSEs of Auto-Encoder-II are always smaller than those of Auto-Encoder-I. Besides, the average MSE of Auto-Encoder-II reduce by 53.66\% compared with Auto-Encoder-I. This table illustrates that the airfoils re-built by Auto-Encoder-II are more accurate and realistic than those re-built by Auto-Encoder-I. Fig. \ref{fig_airfoil_results} depicts the above five airfoils re-built by the two Auto-Encoders. Although we used a 10-degree Bézier curve to smooth all re-built airfoils, all the airfoils re-built by Auto-Encoder-I are not smooth, which means that they cannot be applied in the field of airfoil design. On the contrary, the airfoils re-built by Auto-Encoder-II are closer to the real airfoils and more smooth than those re-built by Auto-Encoder-I. Experiments I demonstrates that Riemannian metrics, as a sort of manifold-feature, can better reflect geometrical nature of airfoils than coordinates.

\subsection{Experiments II: Aerodynamic Performance Prediction Experiments}

In this experiment, we compared multiple methods in terms of discrepant data fusion and aerodynamic performances predictions. The methods are MTL\_g, MTL\_c \cite{hu2022aerodynamic}, RBF-GAN \cite{hu2022flow}, PGML \cite{pawar2021physics}, MLP \cite{xin2022surrogate} and MDF.

\subsubsection{Model Settings}

\begin{figure*}[h!]
    \centering
    \subfigure[]{
       \includegraphics[scale=0.4]{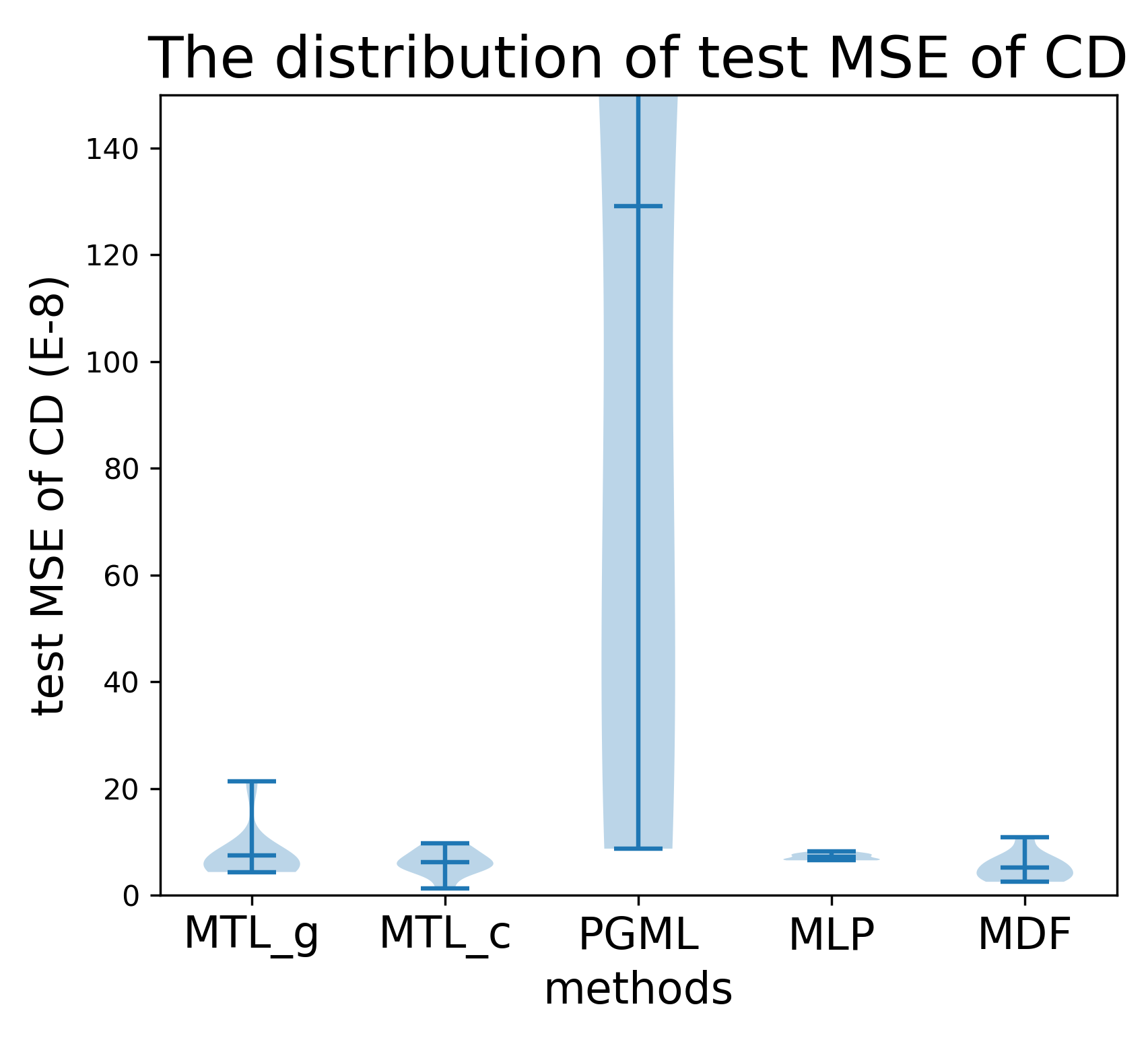}
    }
    \quad
    \subfigure[]{
       \includegraphics[scale=0.4]{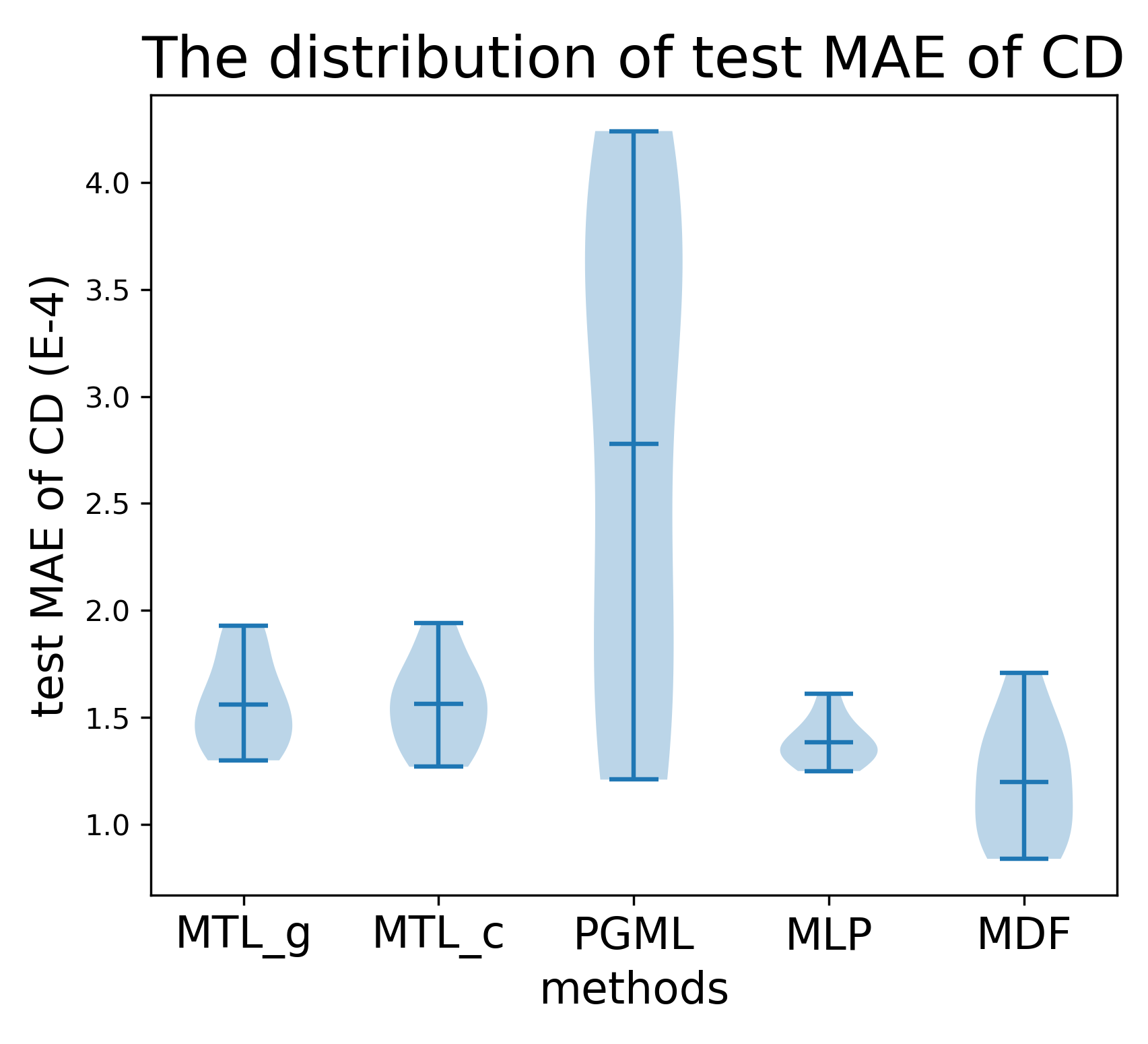}
    }
    \\
    \subfigure[]{
       \includegraphics[scale=0.4]{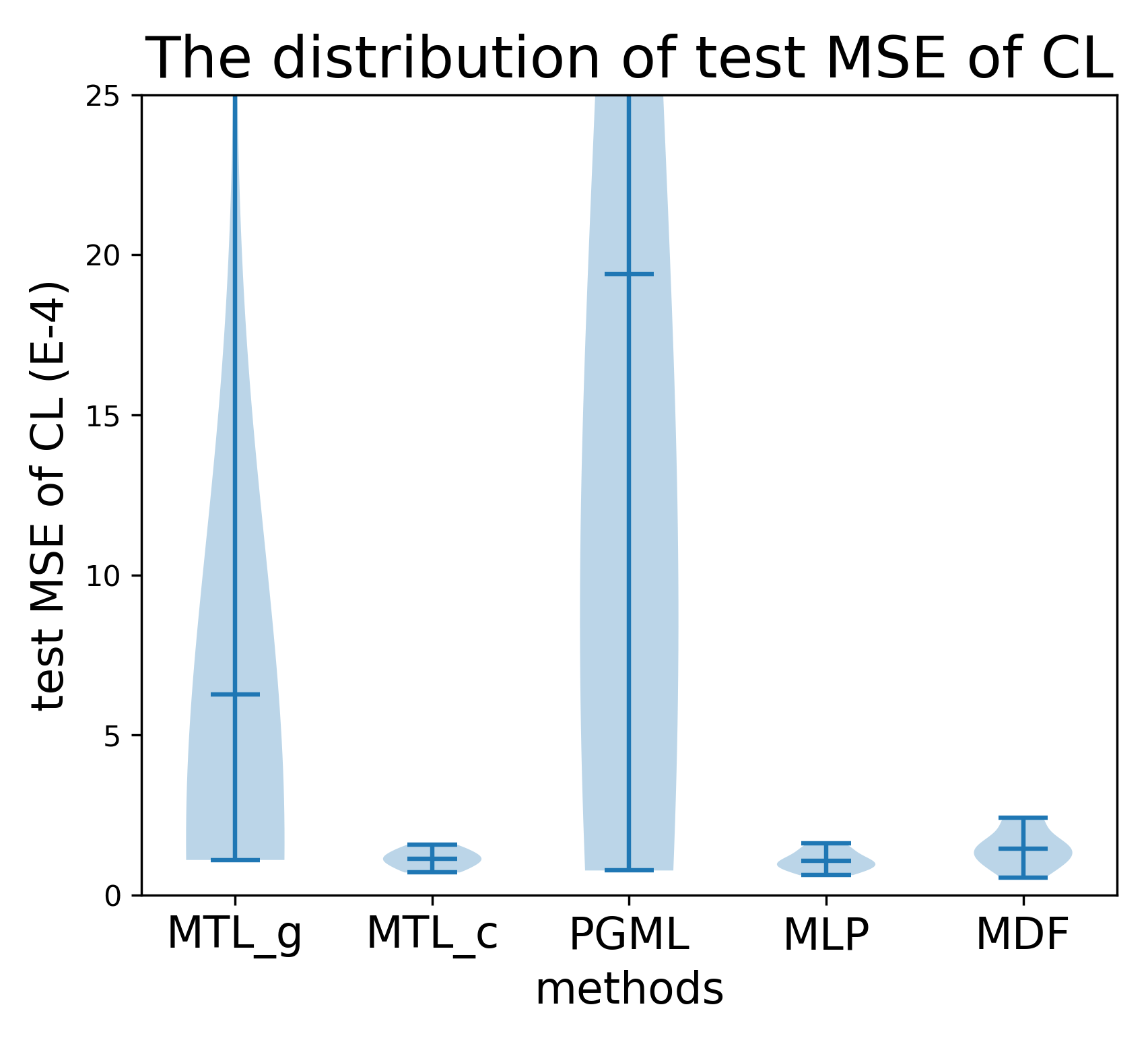}
    }
    \quad
    \subfigure[]{
       \includegraphics[scale=0.4]{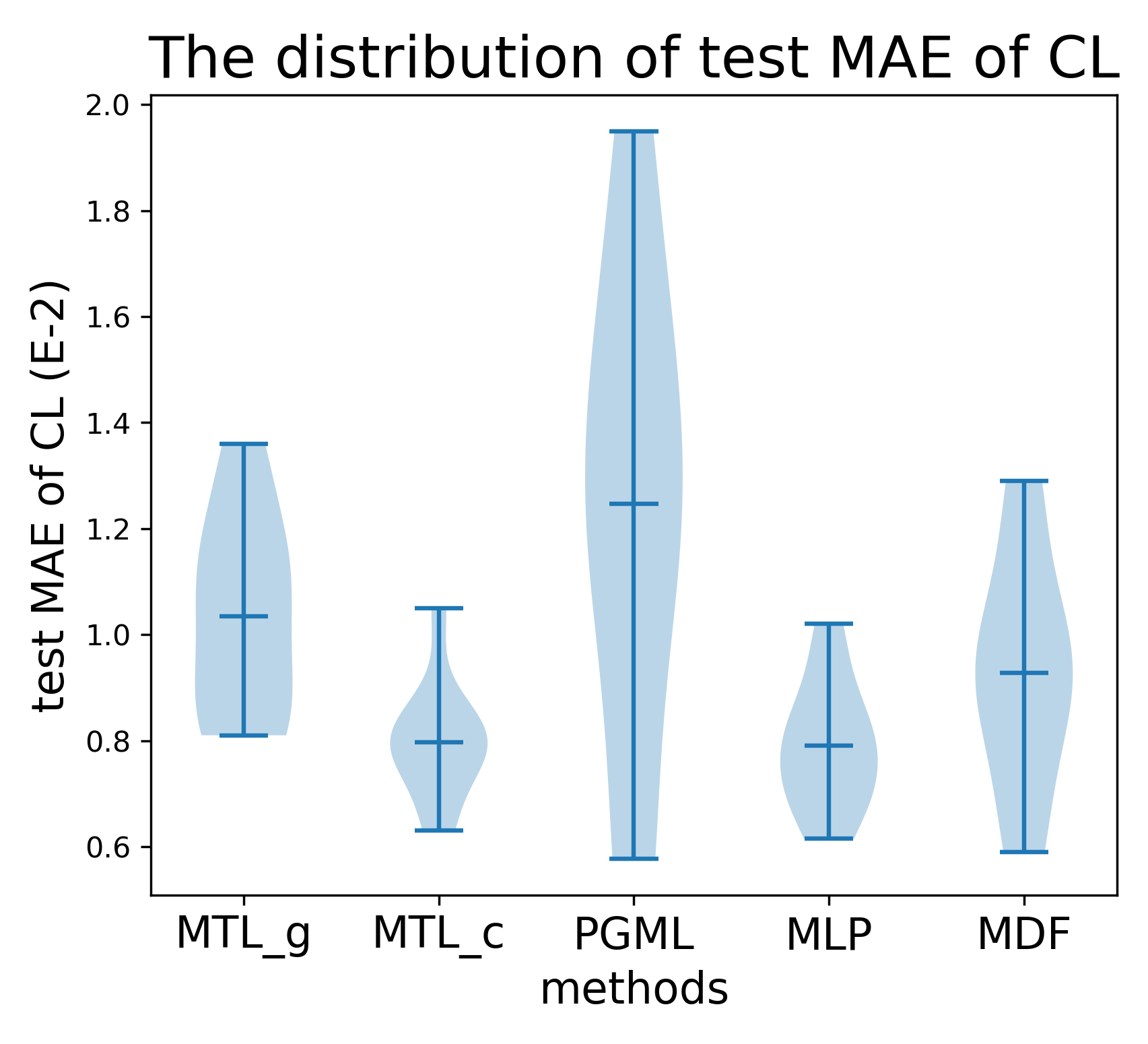}
    }
    \caption{The distributions of MSEs and MAEs of $C_{D}$ and $C_{L}$ predicted by MTL\_g, MTL\_c, PGML, MLP and MDF in 10 round experiments. }
    \label{fig_loss_distribution}
\end{figure*}

All methods in this experiment were repeated 10 rounds, each of which were verified by 10-fold cross validation (i.e., each method is repeated 100 times). The details of hyper-parameters in experiments II are the same as those of experiments I, see Tab.\ref{tab_model_parameter}.

The optimal structure of the above methods are shown in Tab. \ref{tab_structure_MTL}. MTL\_g takes geometrical parameters and flight conditions (Ma and $\alpha$) as inputs. Considering that the geometrical parameters constitute a vector with shape $1 \times 7$, the structure of function network\_1 in MTL\_g is set to ``FCN: 32{*}3'', which indicates that the FCN-based function network\_1 has 3 layers, each of which has 32 neurons. The structure of function network\_2 in MTL\_g is set to ``FCN: 8{*}3'', and the structure of context network in MTL\_g is set to ``FCN: 32{*}3''.

MTL\_c takes coordinates and flight conditions as inputs. Because the coordinates of an airfoil are organized as a matrix with shape $281 \times 2$, the function network\_1 in MTL\_c is designed as a CNN. The CNN-based function network\_1 in MTL\_c consists of three layers of convolutions and one layer of max pooling. The convolution kernels of the three convolution layers have dimensions $2 \times 2 \times 16$, $2 \times 2 \times 32$ and $2 \times 2 \times 64$, respectively. The kernel shape of max pooling layer is set to $2 \times 2 \times 64$. The value of strides of all kernels is 2. Besides, the structure of function network\_2 in MTL\_c is set to ``FCN: 16{*}3'', and the structure of context network in MTL\_c is set to ``FCN: 512{*}3''.

RBF-GAN whose RBF-based discriminator has been proved to be the optimal approximation to a discriminator can capture more accurate details of flow fields and further to predict the variations of flow fields \cite{hu2022flow}. Therefore, RBF-GAN is chosen as a comparative method to take both geometric-features and flight conditions as inputs, so as to predict aerodynamic performances. Considering that the discriminator of RBF-GAN is a RBFNN, the coordinate matrix with shape $281*2$ is expanded into a vector with shape $562*1$. The input of RBF-based discriminator is the combination of expanded coordinate vector and flight condition vector. The structure of RBF-based discriminator is set to ``RBF: 512{*}1''. In addition, according to \cite{hu2022flow}, the structure of corresponding generator is set to ``FCN: 1024{*}3''.

PGML is used to fuse the coordinates of airfoils and flight conditions in \cite{pawar2021physics}. PGML consists of two FCN-based subnetworks. The first subnetwork takes airfoil coordinates as inputs, and the structure is set to ``FCN: 1024{*}3''. The second subnetwork takes the flight conditions as inputs, and the structure is set to ``FCN: 1024{*}1''. The last layer of the two subnetworks is concatenated by a FCN to enable the fusion of coordinates of airfoils and flight conditions.

MLP is a complete FCN whose input is the combination of airfoil coordinates and flight conditions. MLP is the most intuitive fusion method that takes the combination of airfoils coordinates and flight conditions as inputs. The purpose of selecting MLP as a compartive method is to compare this intuitive fusion method with MDF. The structure of MLP is set to ``FCN: 1024{*}3''.

MDF takes manifold-features and flight conditions as inputs. The input manifold-features are organized as a vector with shape $1 \times 271$, therefore, the function network\_1 in MDF is designed as a FCN. The structure of FCN-based function network\_1 in MDF is set to ``FCN: 1024{*}3''. In addition, the structure of function network\_2 in MDF is set to ``FCN: 16{*}3'', and the structure of context network in MDF is set to ``FCN: 512{*}3''.

The MSE and the mean absolute error (MAE) are chosen as assessment indicators for predicted aerodynamic performances:
\begin{equation}
\label{equ_mse_mae_performance}
\nonumber
\begin{cases}
  MSE = \frac{1}{N} \sum_{z=1}^{N}\left(\frac{1}{K} \sum_{i=1}^{K}\left(y_{z i}-\hat{y}_{z i}\right)^{2}\right) \\
  MAE = \frac{1}{N} \sum_{z=1}^{N}\left(\frac{1}{K} \sum_{i=1}^{K} ||y_{z i}-\hat{y}_{z i} || \right)
\end{cases}
\end{equation}
where $y_{zi}$ denotes the predicted performance of the $z$th input and $\hat{y}_{zi}$ denotes the true performance of the $z$th input in the test set.

\subsubsection{Results and Analyses}

\begin{table*}[h!]
\caption{The average MSE and MAE of MTL\_g, MTL\_c, RBF-GAN, PGML, MLP and MDF.}
\label{tab_MTL_errors}
\centering
  \begin{tabular}{ccccc}
    \hline
    models   & MSE of $C_{D}$         & MAE of $C_{D}$         & MSE of $C_{L}$         & MAE of $C_{L}$         \\ \hline
    MTL\_g   & $7.41 \times 10^{-8}$  & $1.57 \times 10^{-4}$  & $6.27 \times 10^{-4}$  & $1.03 \times 10^{-2}$  \\
    MTL\_c \cite{hu2022aerodynamic} & $6.77 \times 10^{-8}$  & $1.56 \times 10^{-4}$    & $1.13 \times 10^{-4}$  & $8.15 \times 10^{-3}$ \\
    RBF-GAN \cite{hu2022flow}  & $1.30 \times 10^{-5}$      & $1.96 \times 10^{-3}$     & $3.42 \times 10^{-2}$            & $9.34 \times 10^{-2}$ \\

    PGML \cite{pawar2021physics}& $1.29 \times 10^{-6}$     & $2.78 \times 10^{-4}$     & $1.67 \times 10^{-3}$            &$1.35 \times 10^{-2}$  \\

    MLP \cite{xin2022surrogate} &  $7.29 \times 10^{-8}$    & $1.38 \times 10^{-4}$     & $\mathbf{1.07 \times 10^{-4}}$            & $\mathbf{7.90 \times 10^{-3}}$  \\
    MDF (our method) & $\mathbf{3.80 \times 10^{-8}}$  & $\mathbf{1.02 \times 10^{-4}}$  & $1.21 \times 10^{-4}$  & $8.62 \times 10^{-3}$  \\
    error reduction (MDF v.s. MLP) & $35.37\%$                   & $19.05\%$                  & $-13.08\%$                      &$-9.11\%$\\
    \hline
  \end{tabular}
\end{table*}

\begin{figure*}[h!]
    \centering
    \subfigure[]{
       \includegraphics[scale=0.33]{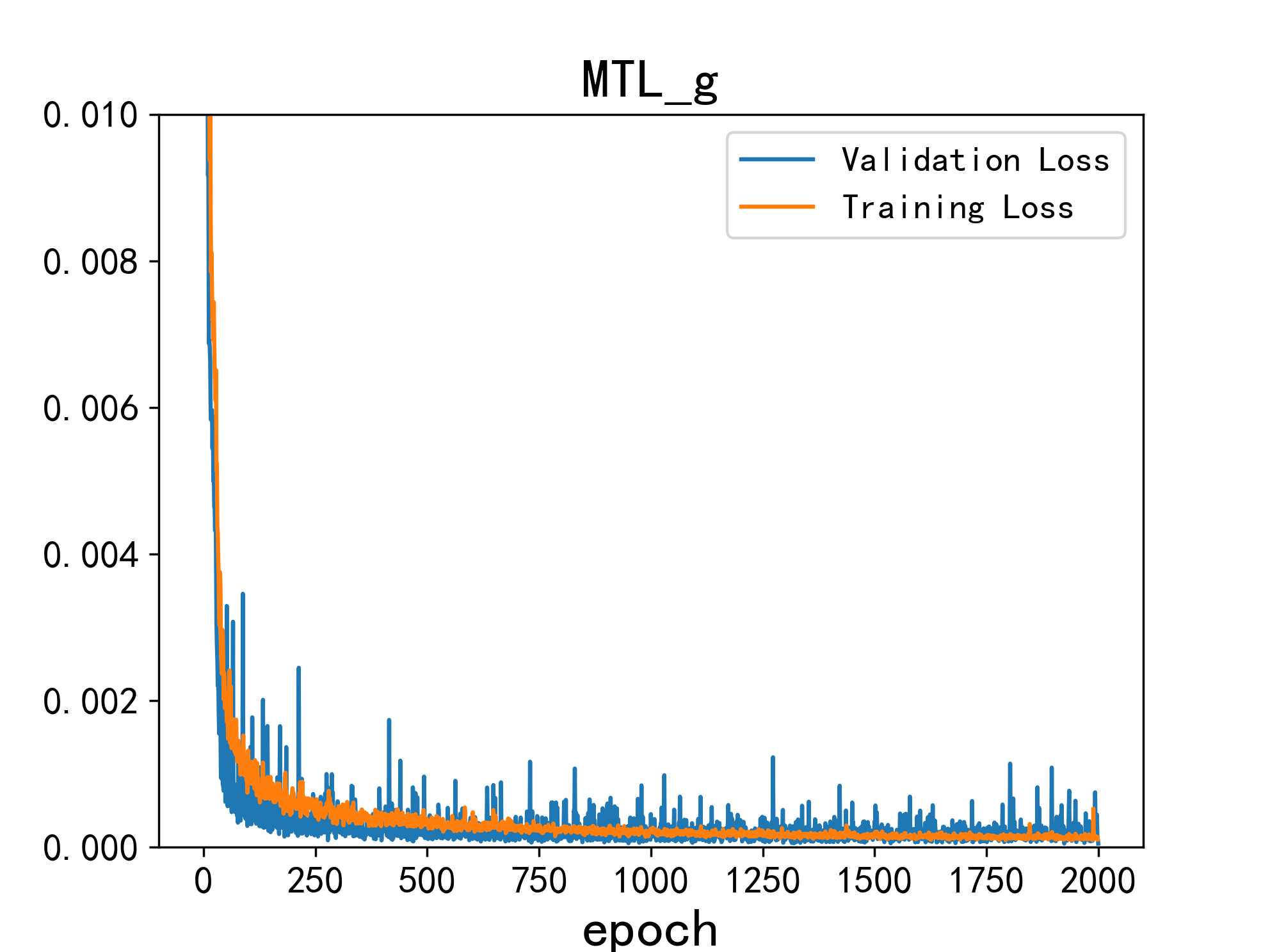}
    }
    \quad
    \subfigure[]{
       \includegraphics[scale=0.33]{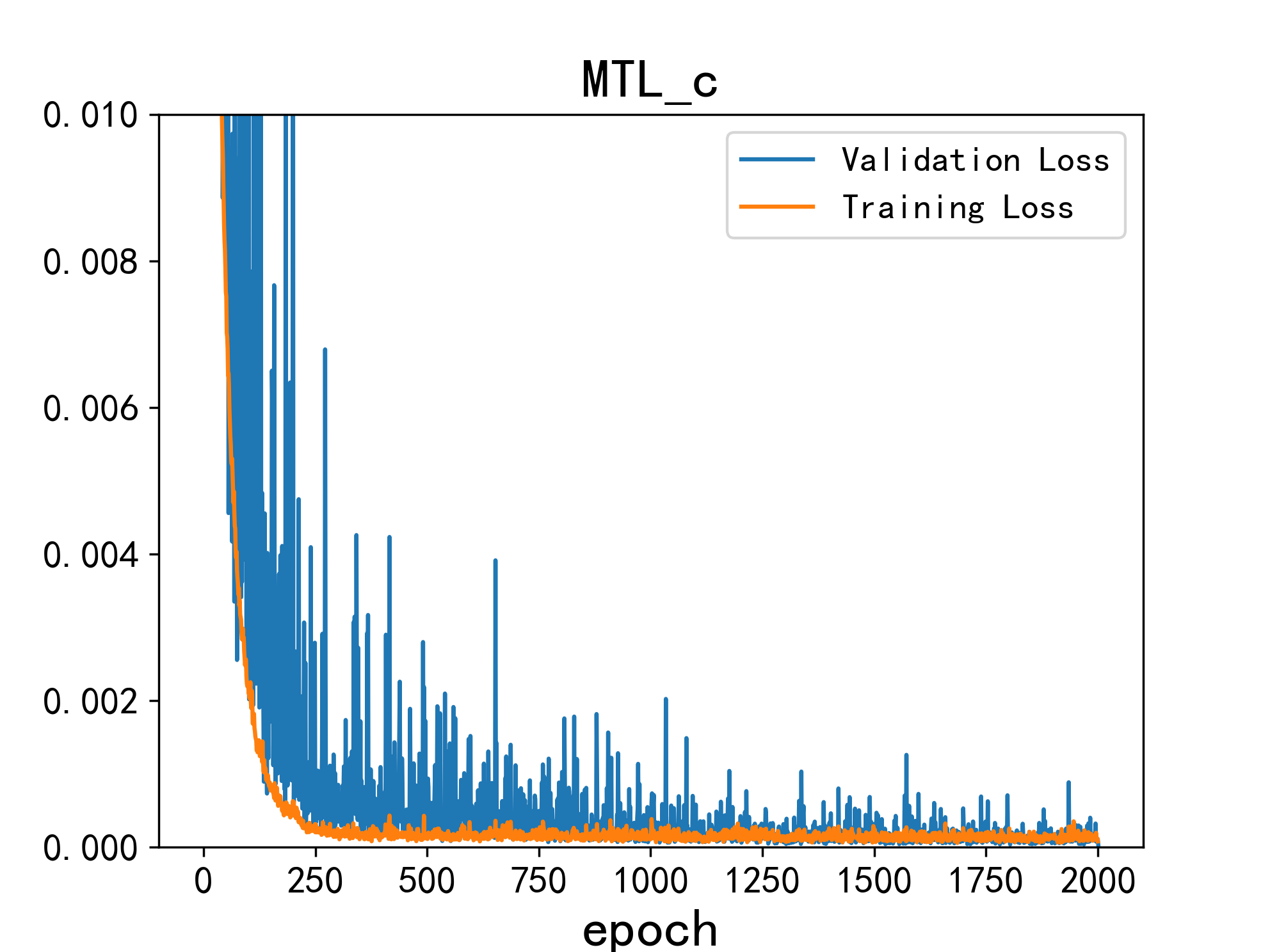}
    }
    \quad
    \subfigure[]{
       \includegraphics[scale=0.33]{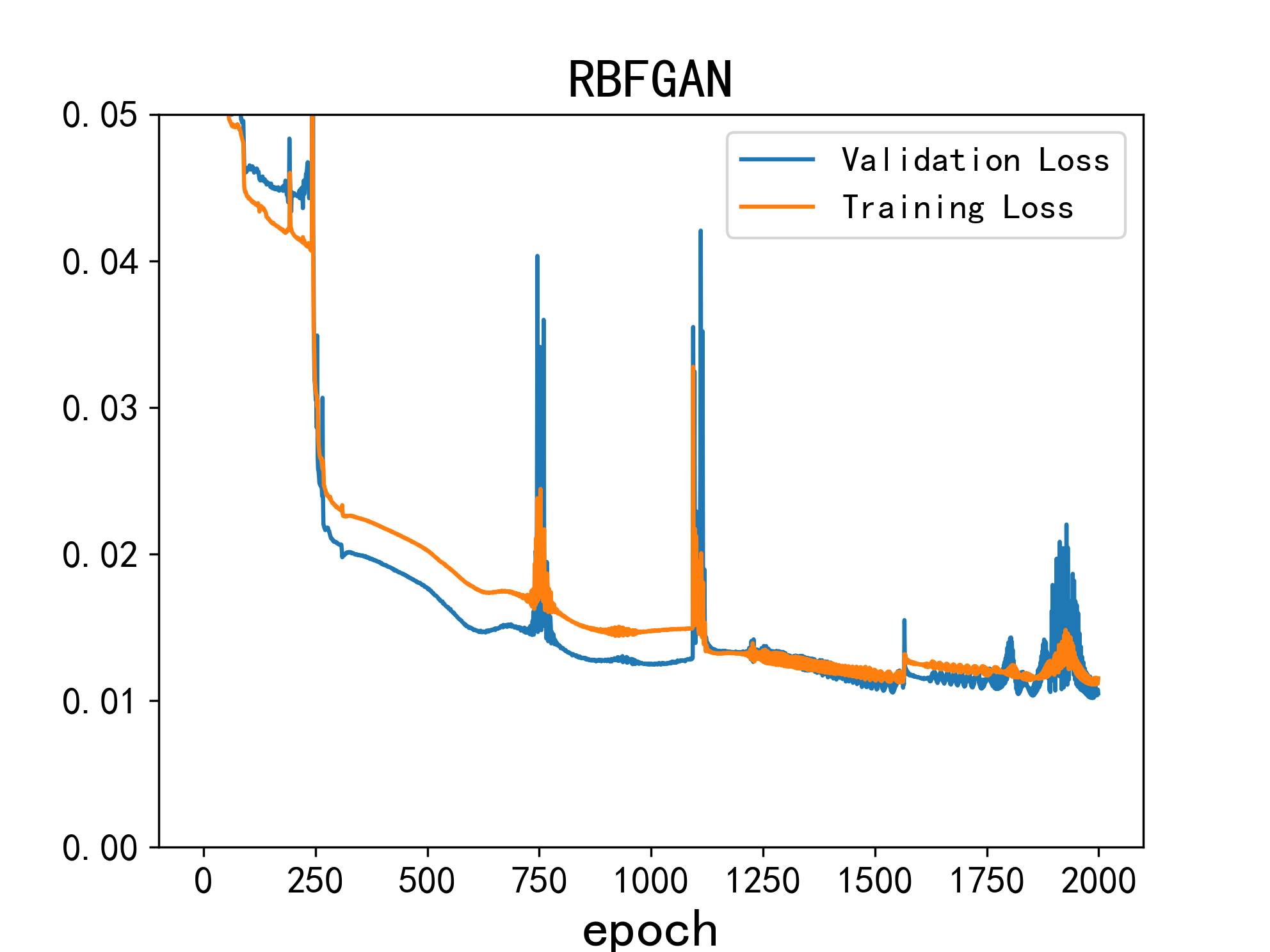}
    }
    \quad
    \subfigure[]{
       \includegraphics[scale=0.33]{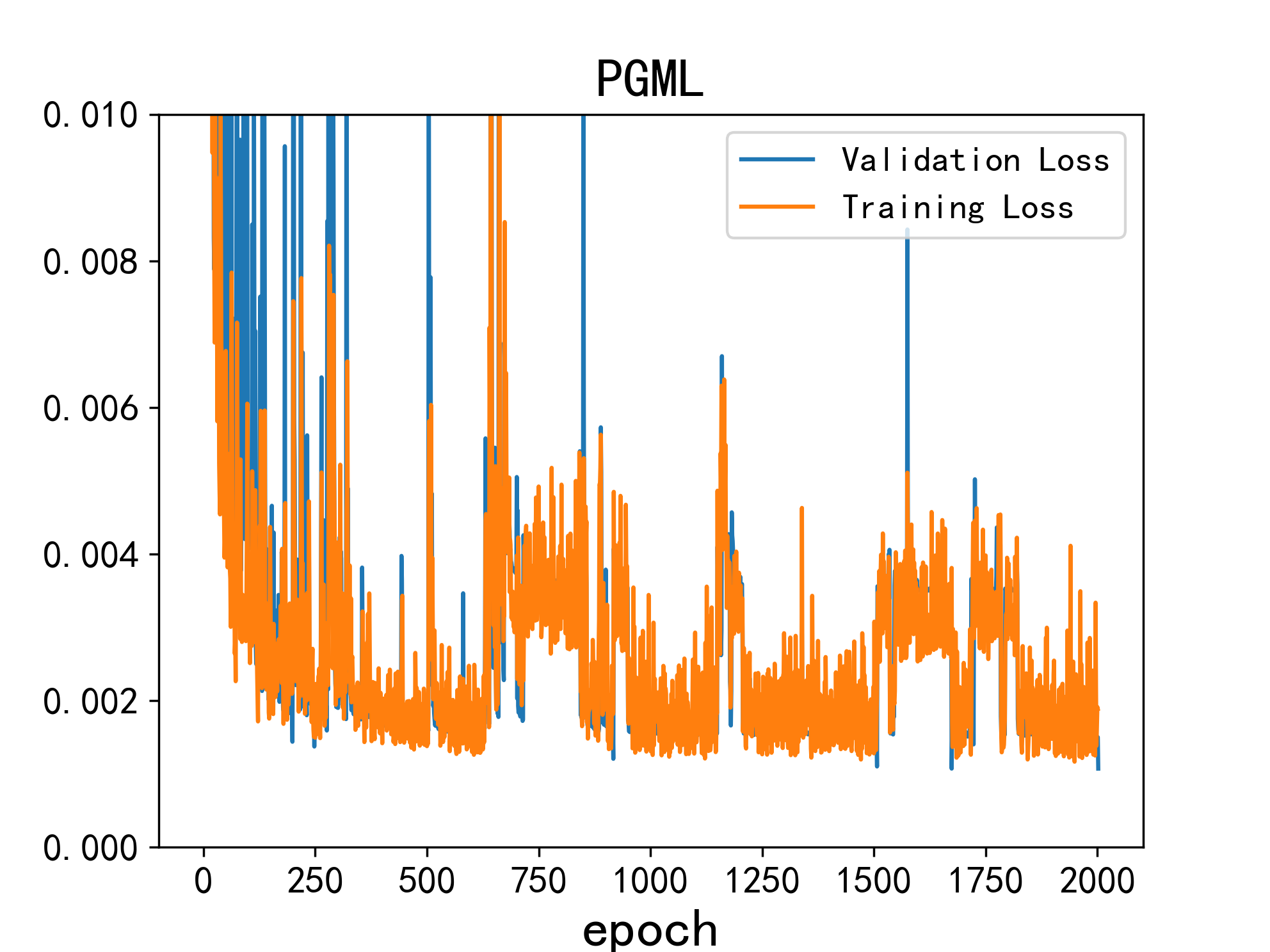}
    }
    \quad
    \subfigure[]{
       \includegraphics[scale=0.33]{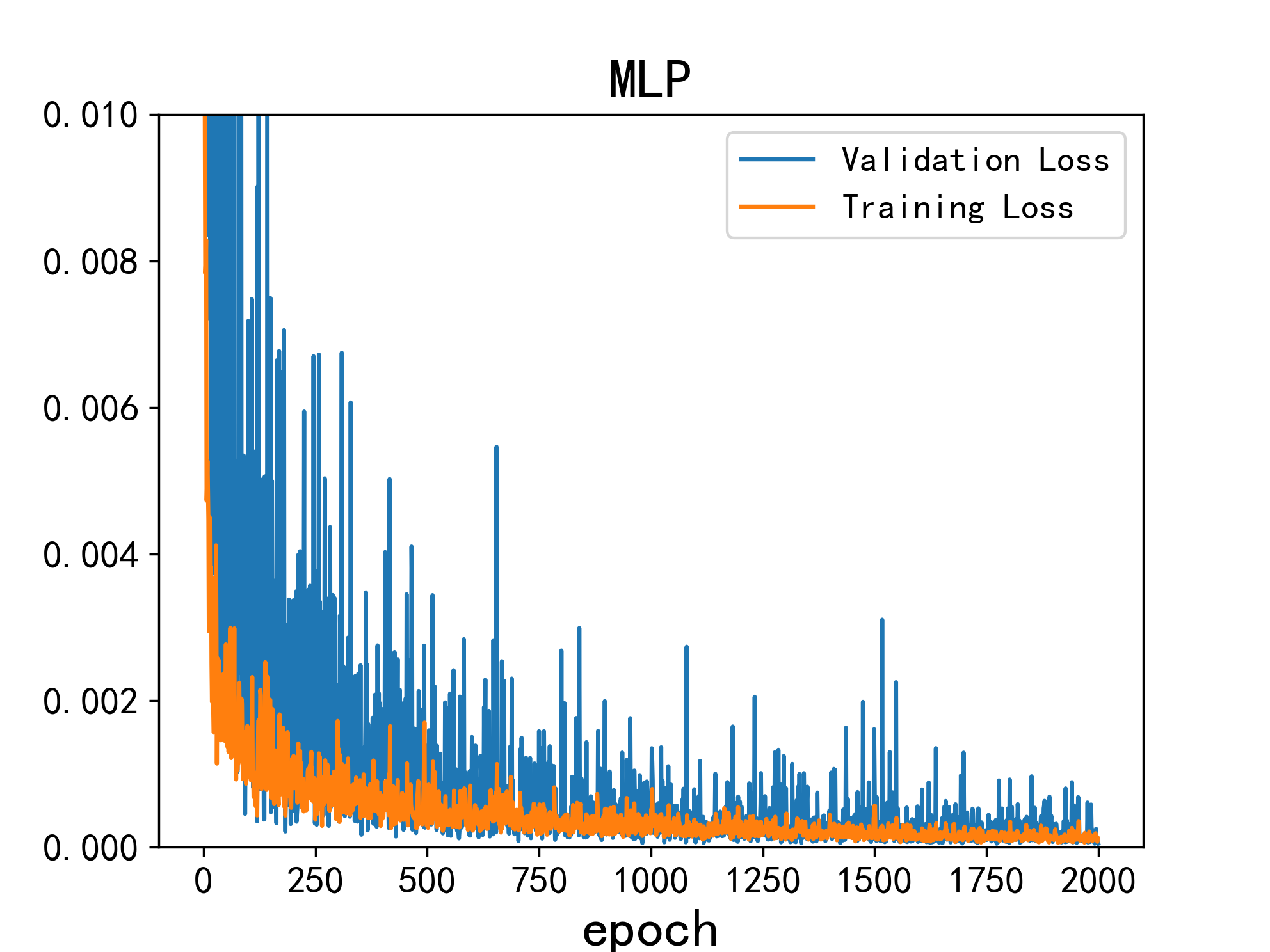}
    }
    \quad
    \subfigure[]{
       \includegraphics[scale=0.33]{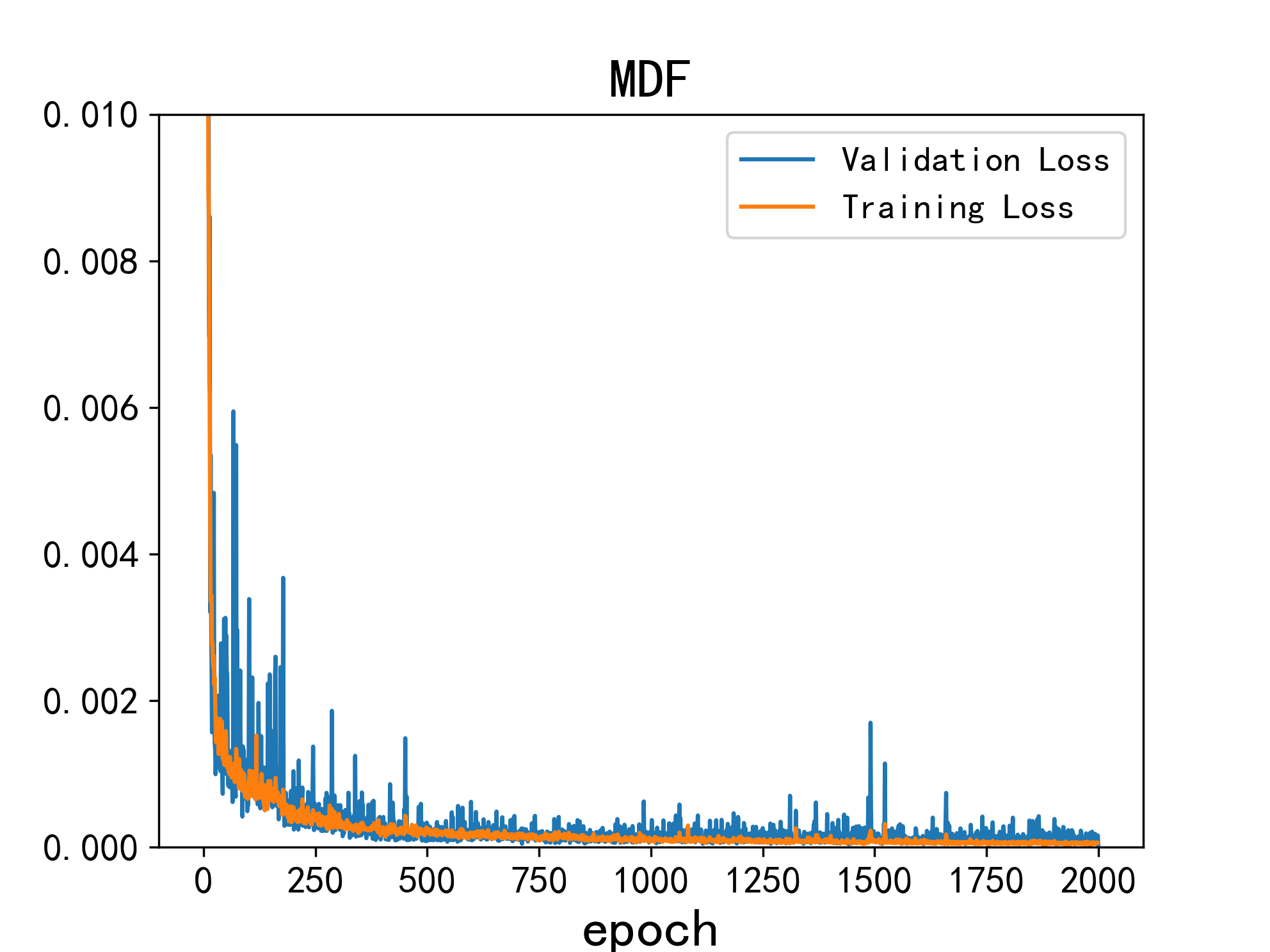}
    }
    \caption{The training loss and validation loss of methods compared in Experiments II in round 2. }
    \label{fig_loss_2}
\end{figure*}

\begin{figure*}[h!]
    \centering
    \subfigure[]{
       \includegraphics[scale=0.33]{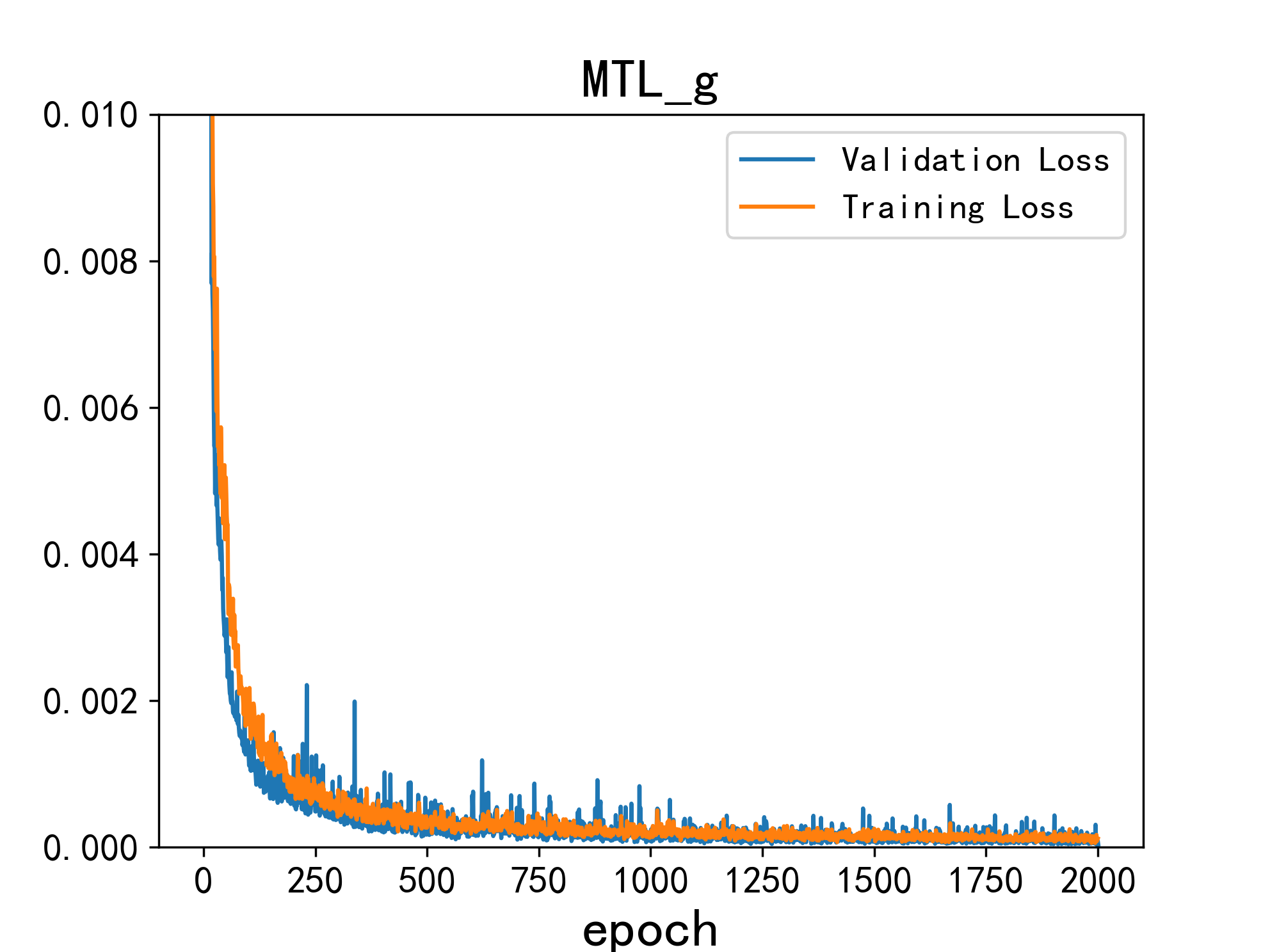}
    }
    \quad
    \subfigure[]{
       \includegraphics[scale=0.33]{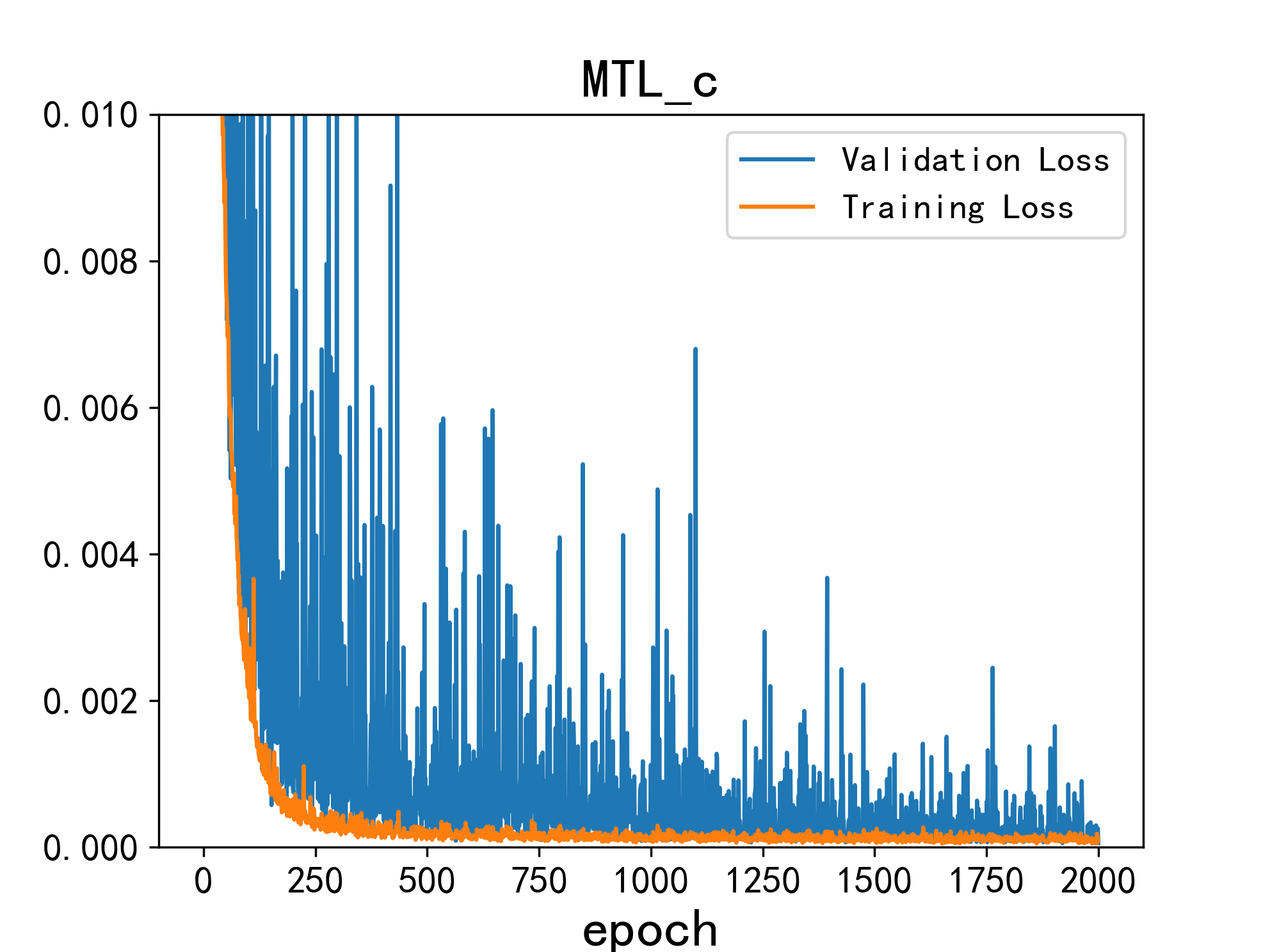}
    }
    \quad
    \subfigure[]{
       \includegraphics[scale=0.33]{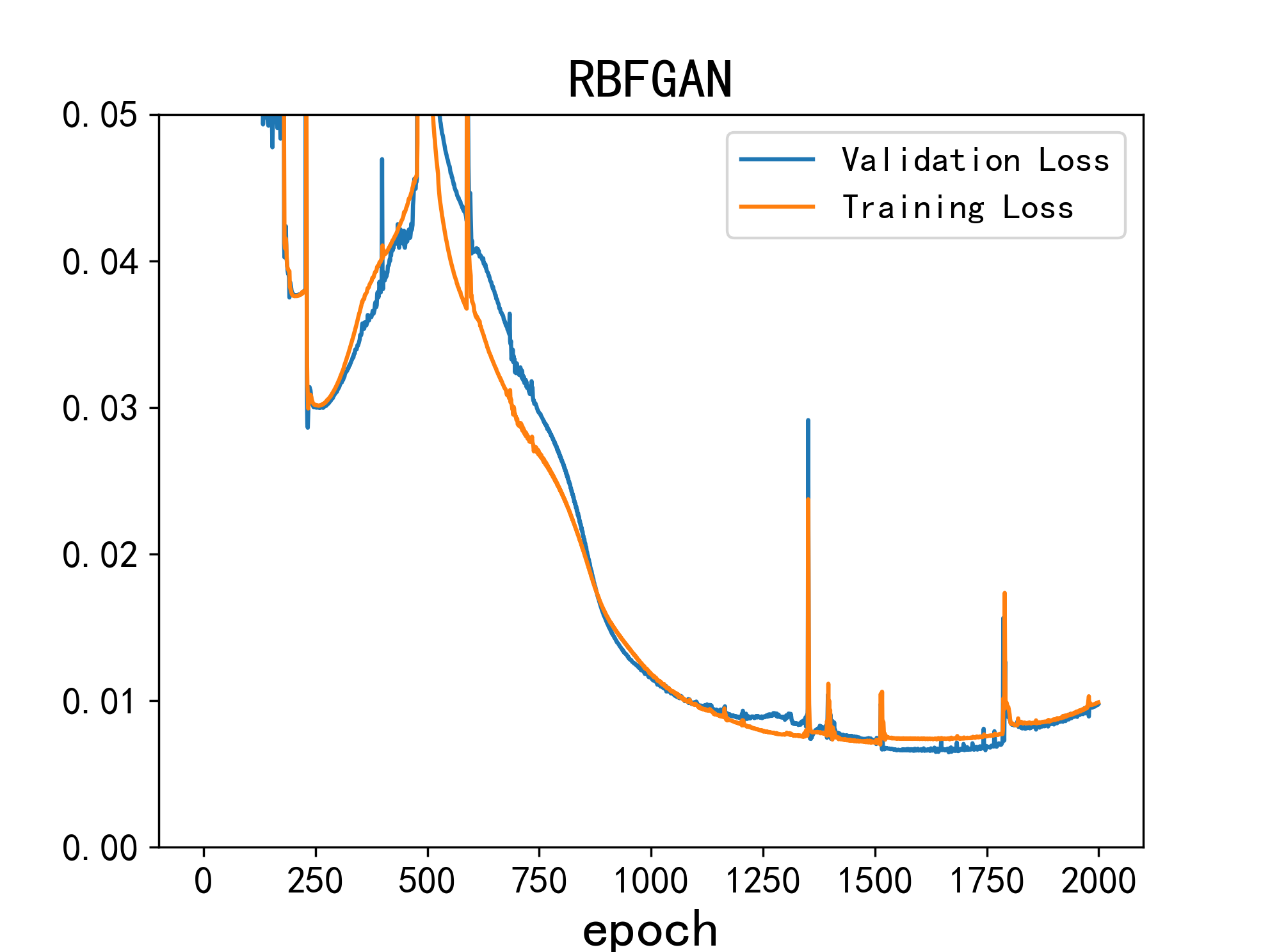}
    }
    \quad
    \subfigure[]{
       \includegraphics[scale=0.33]{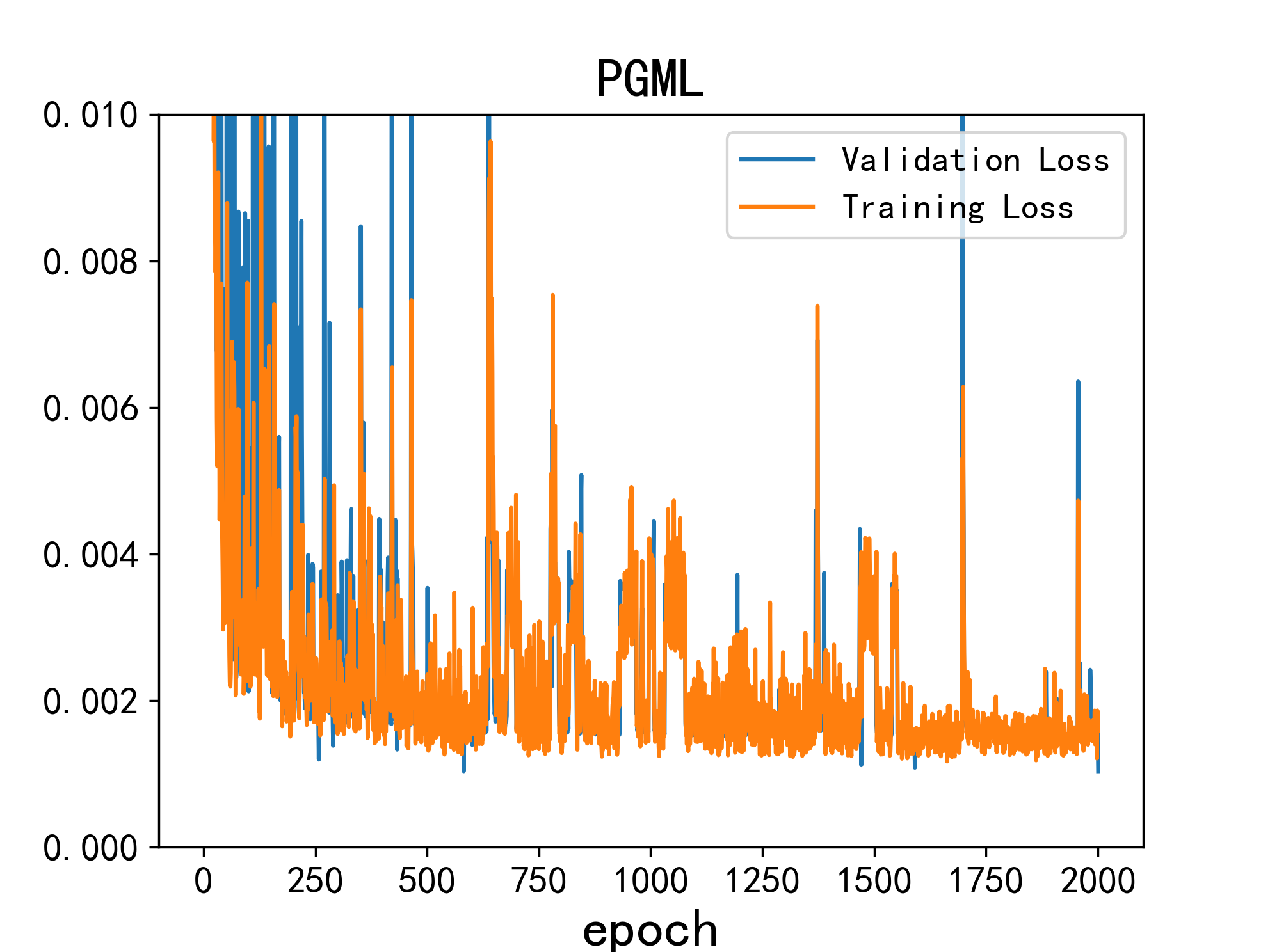}
    }
    \quad
    \subfigure[]{
       \includegraphics[scale=0.33]{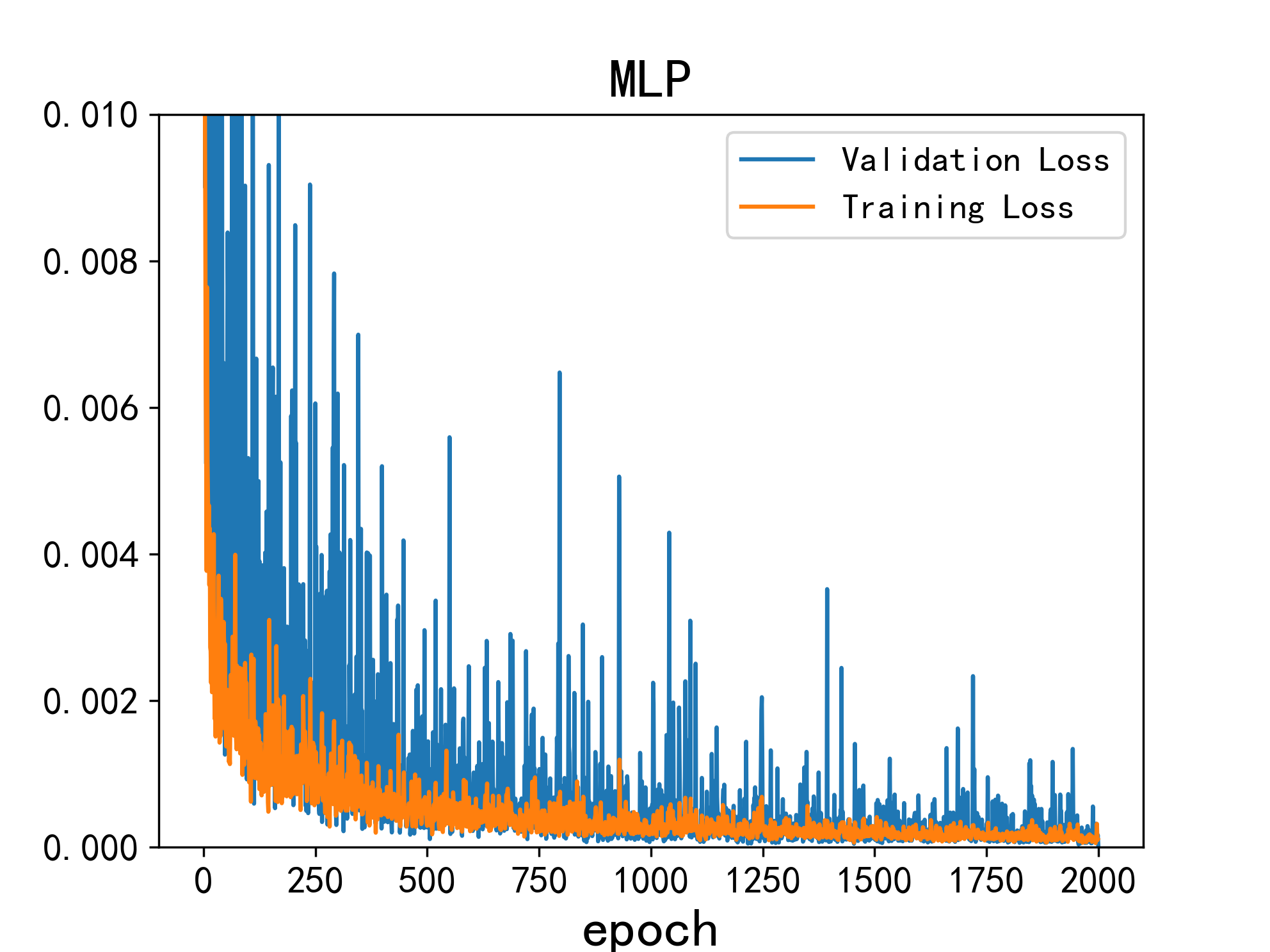}
    }
    \quad
    \subfigure[]{
       \includegraphics[scale=0.33]{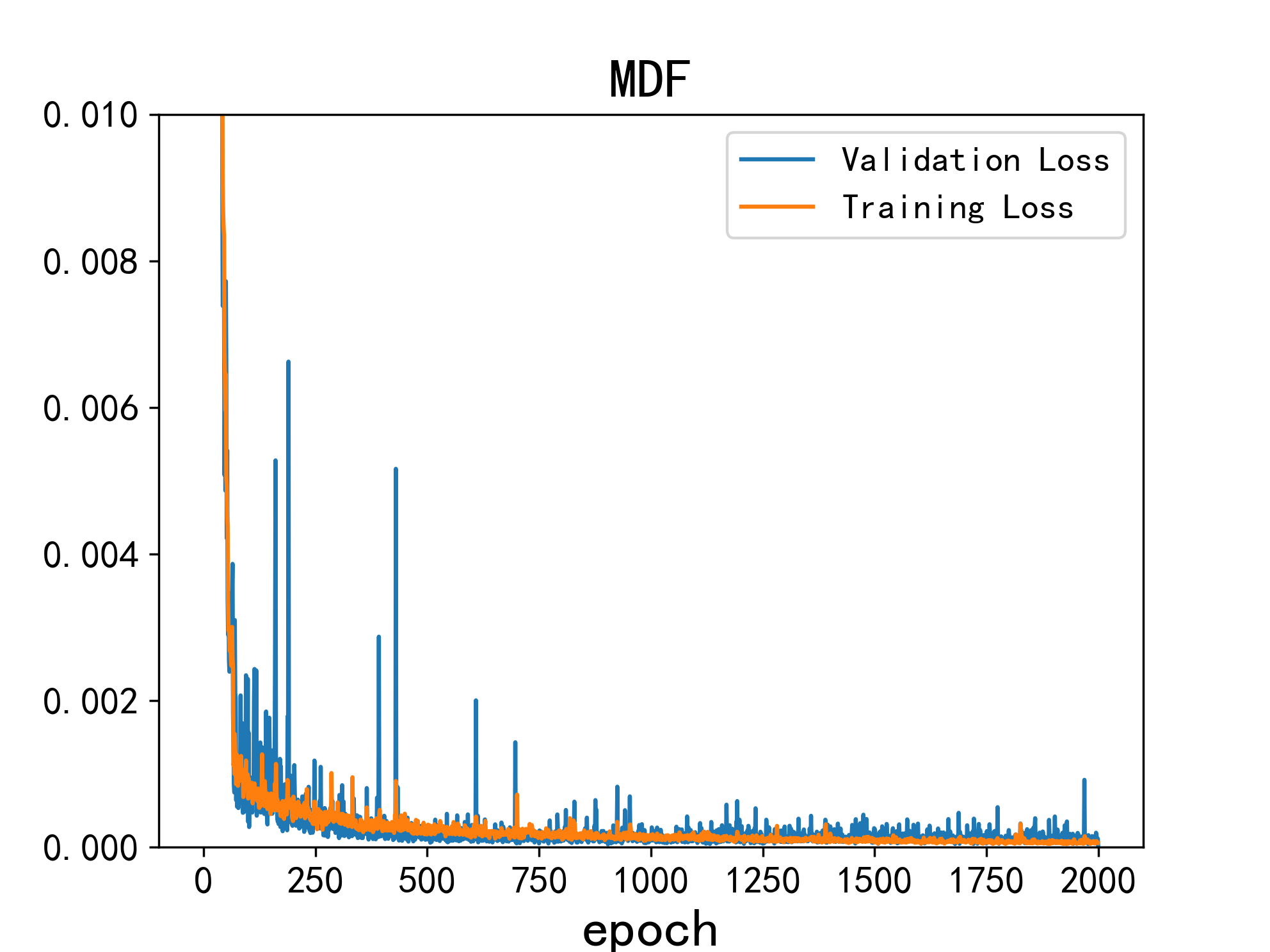}
    }
    \caption{The training loss and validation loss of methods compared in Experiments II in round 9. }
    \label{fig_loss_9}
\end{figure*}

\begin{figure*}[h!]
    \centering
    \subfigure[]{
       \includegraphics[scale=0.35]{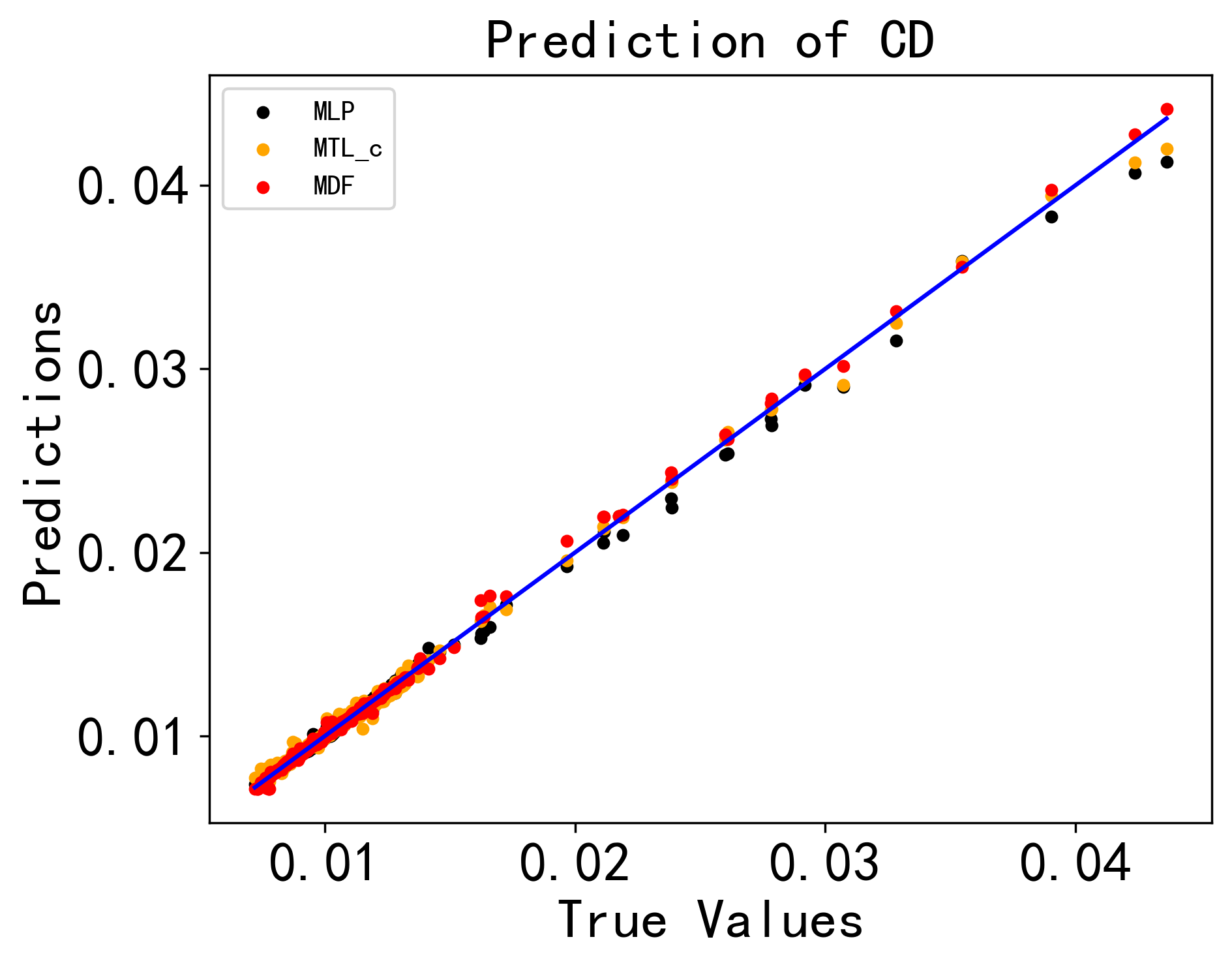}
    }
    \quad
    \subfigure[]{
       \includegraphics[scale=0.35]{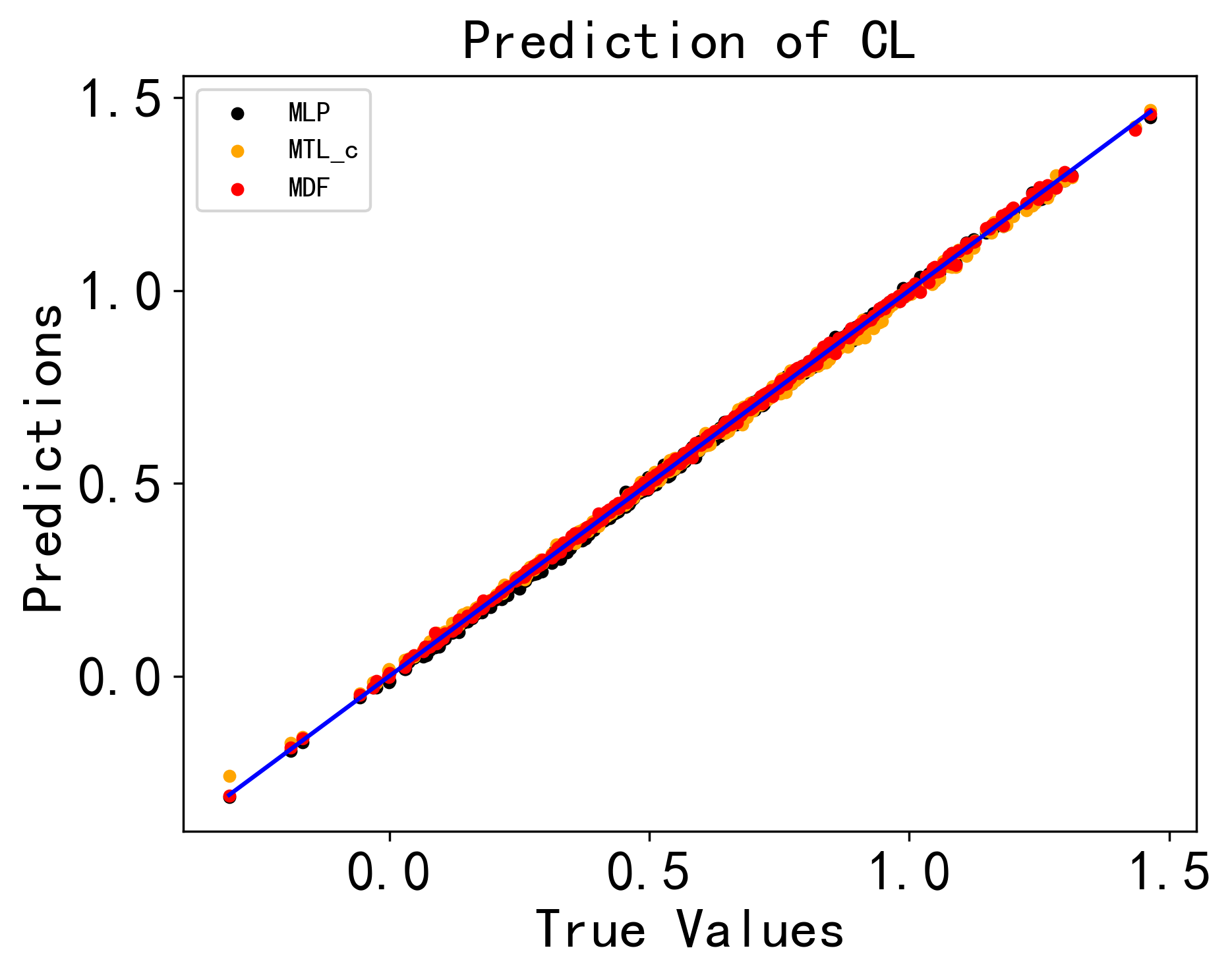}
    }
    \\
    \subfigure[]{
       \includegraphics[scale=0.35]{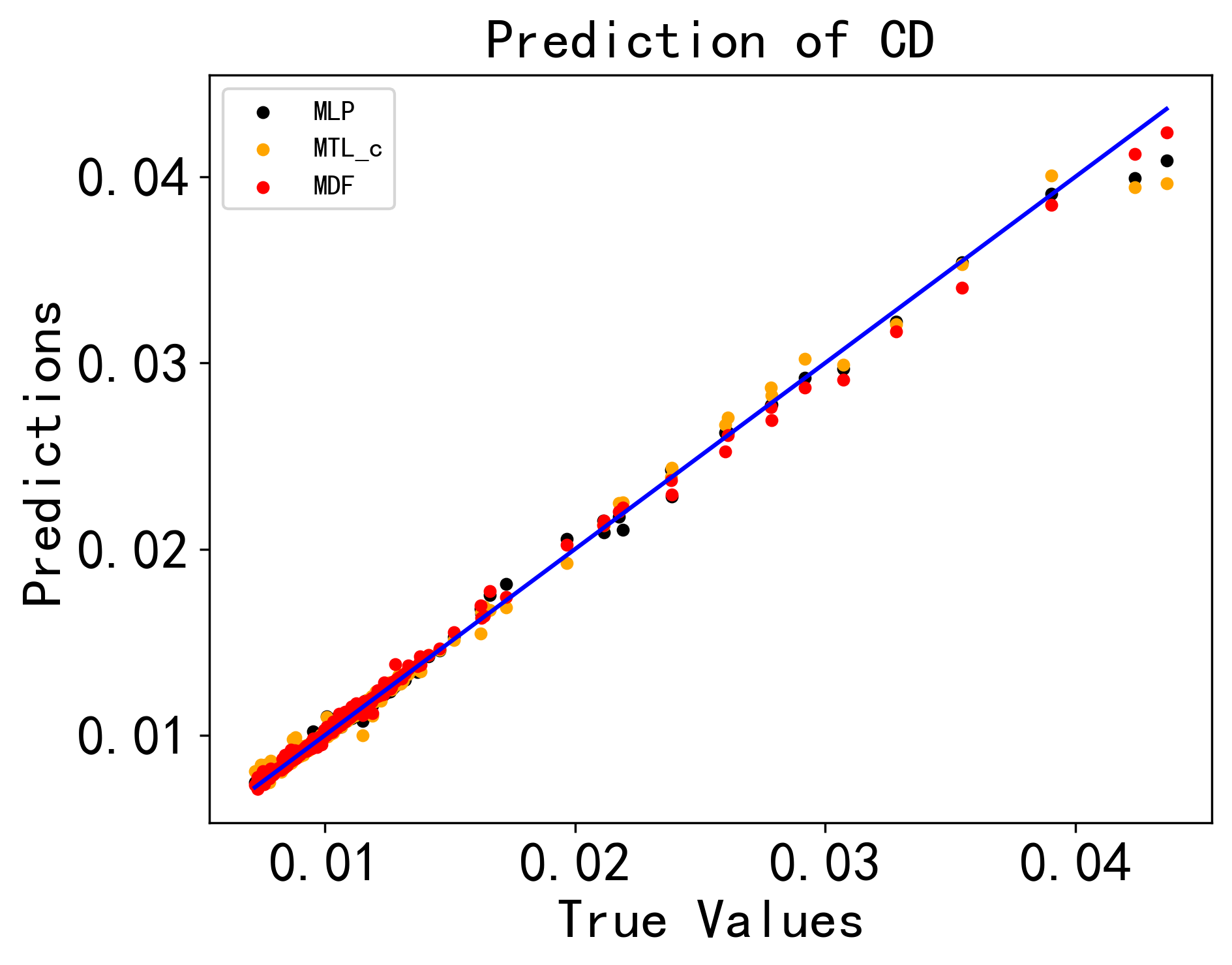}
    }
    \quad
    \subfigure[]{
       \includegraphics[scale=0.35]{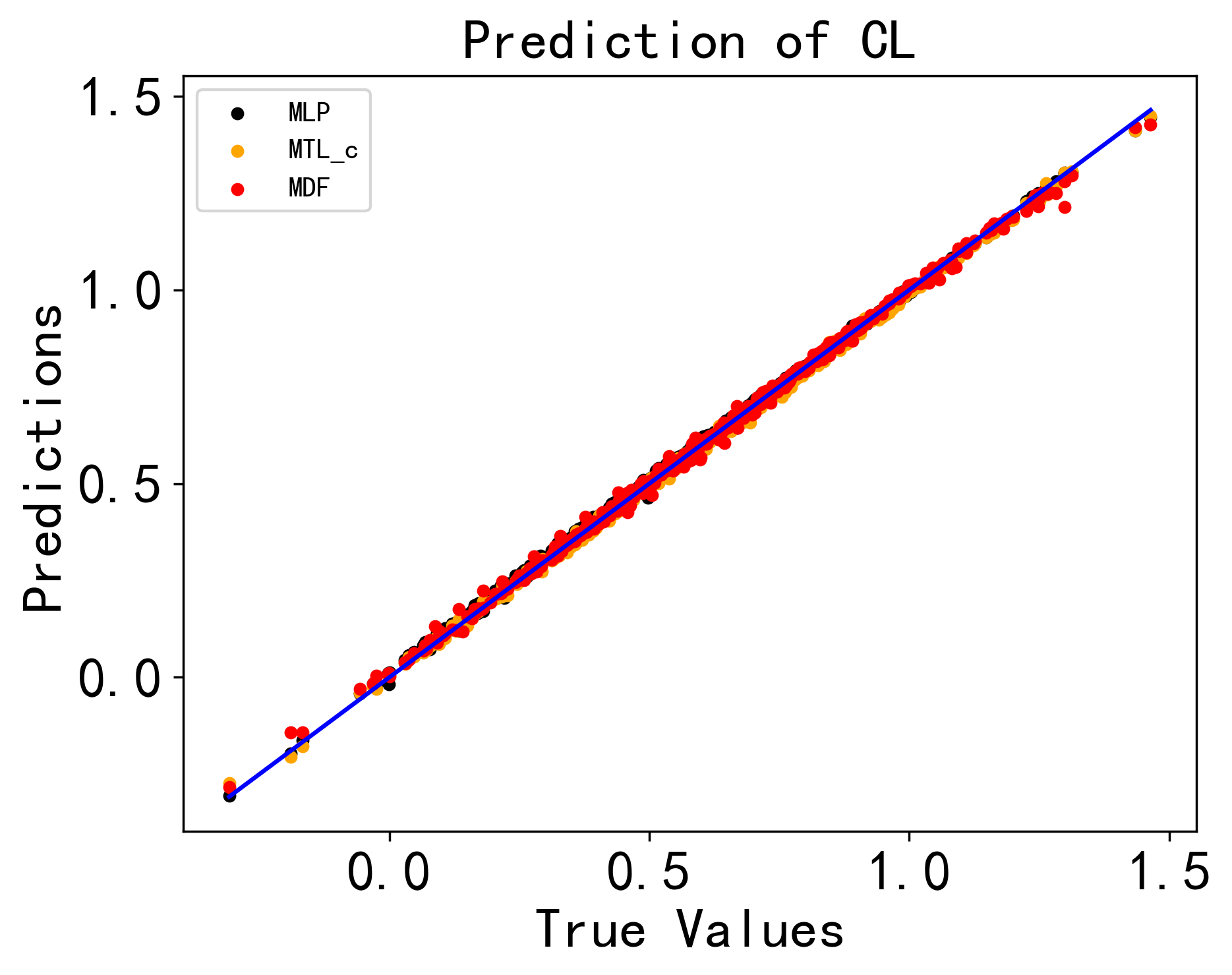}
    }
    \caption{The $C_{D}$ and the $C_{L}$ predicted by MLP, MTL\_c and MDF in two different round experiments. }
    \label{fig_cd_cl_45_line}
\end{figure*}

\begin{figure*}[h!]
    \centering
    \subfigure[]{
       \includegraphics[scale=0.35]{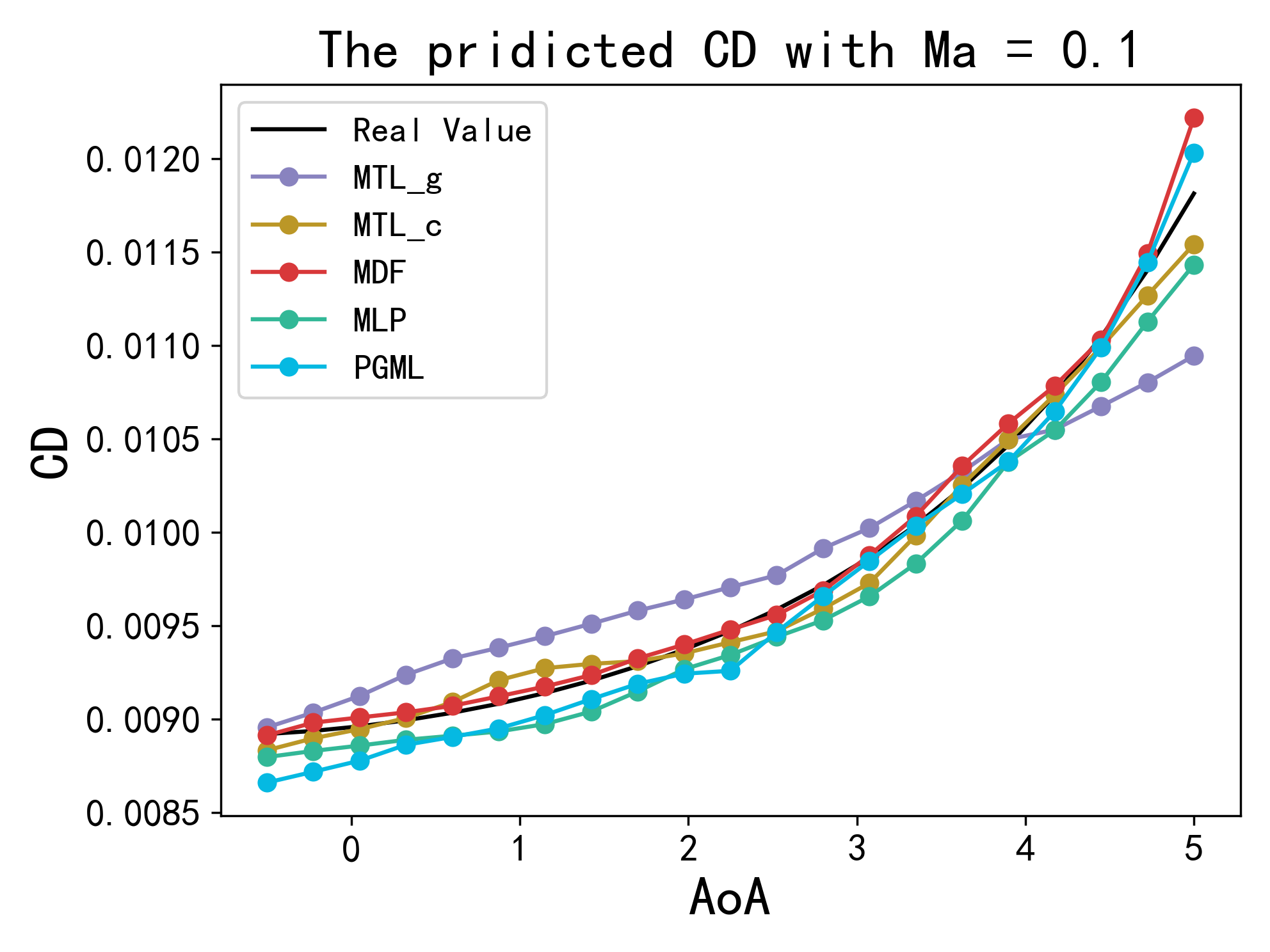}
    }
    \quad
    \subfigure[]{
       \includegraphics[scale=0.35]{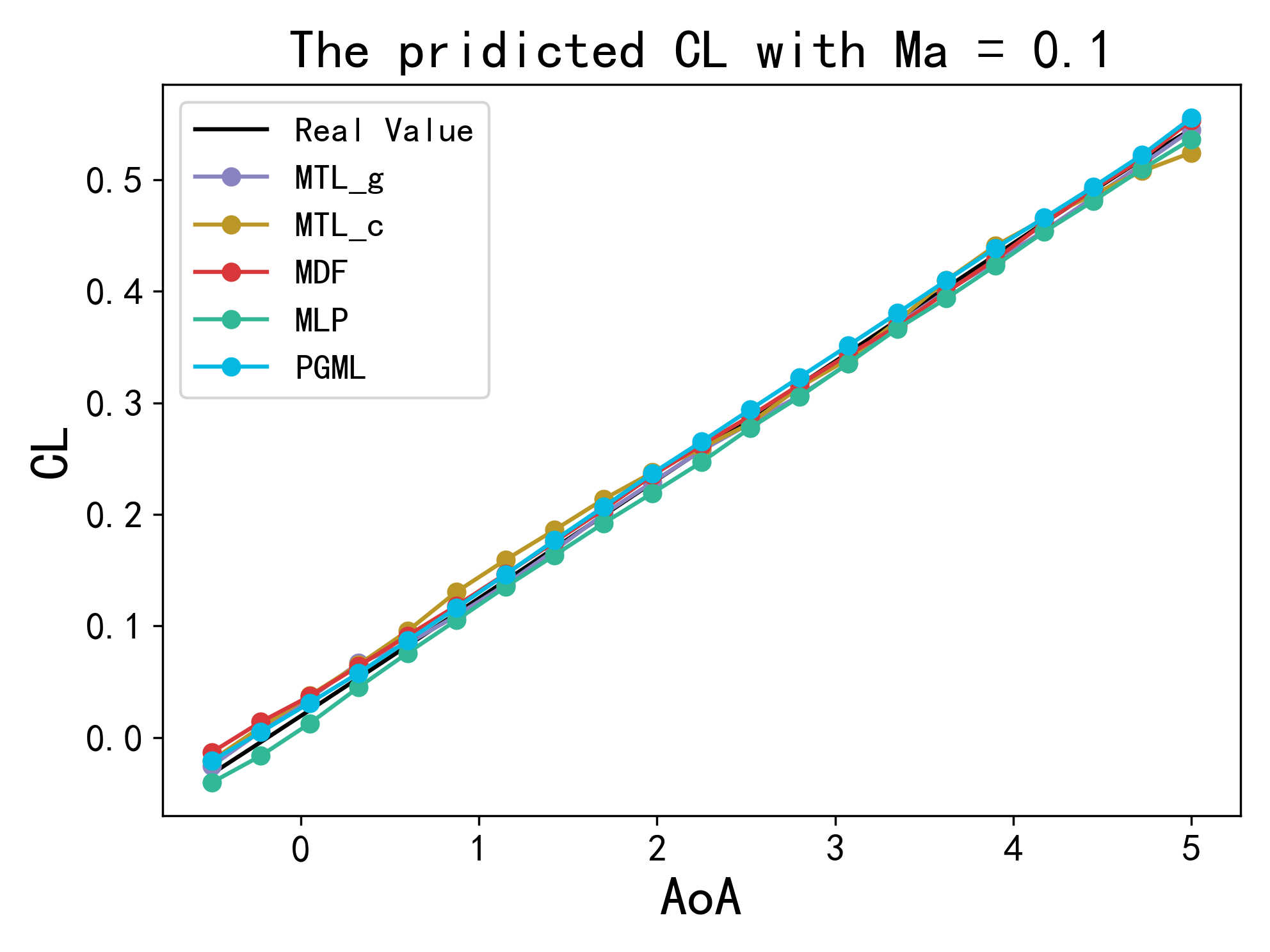}
    }
    \\
    \subfigure[]{
       \includegraphics[scale=0.35]{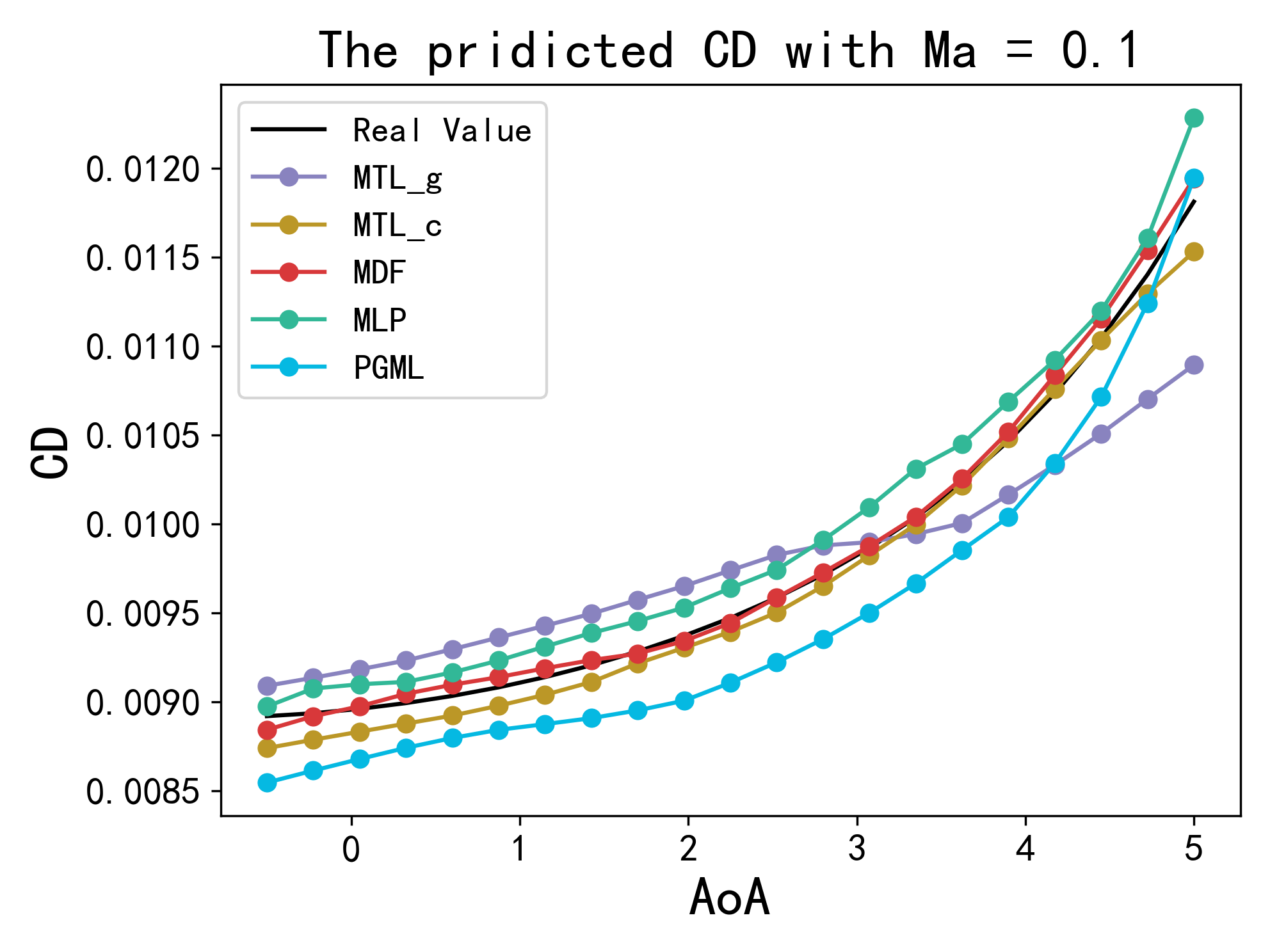}
    }
    \quad
    \subfigure[]{
       \includegraphics[scale=0.35]{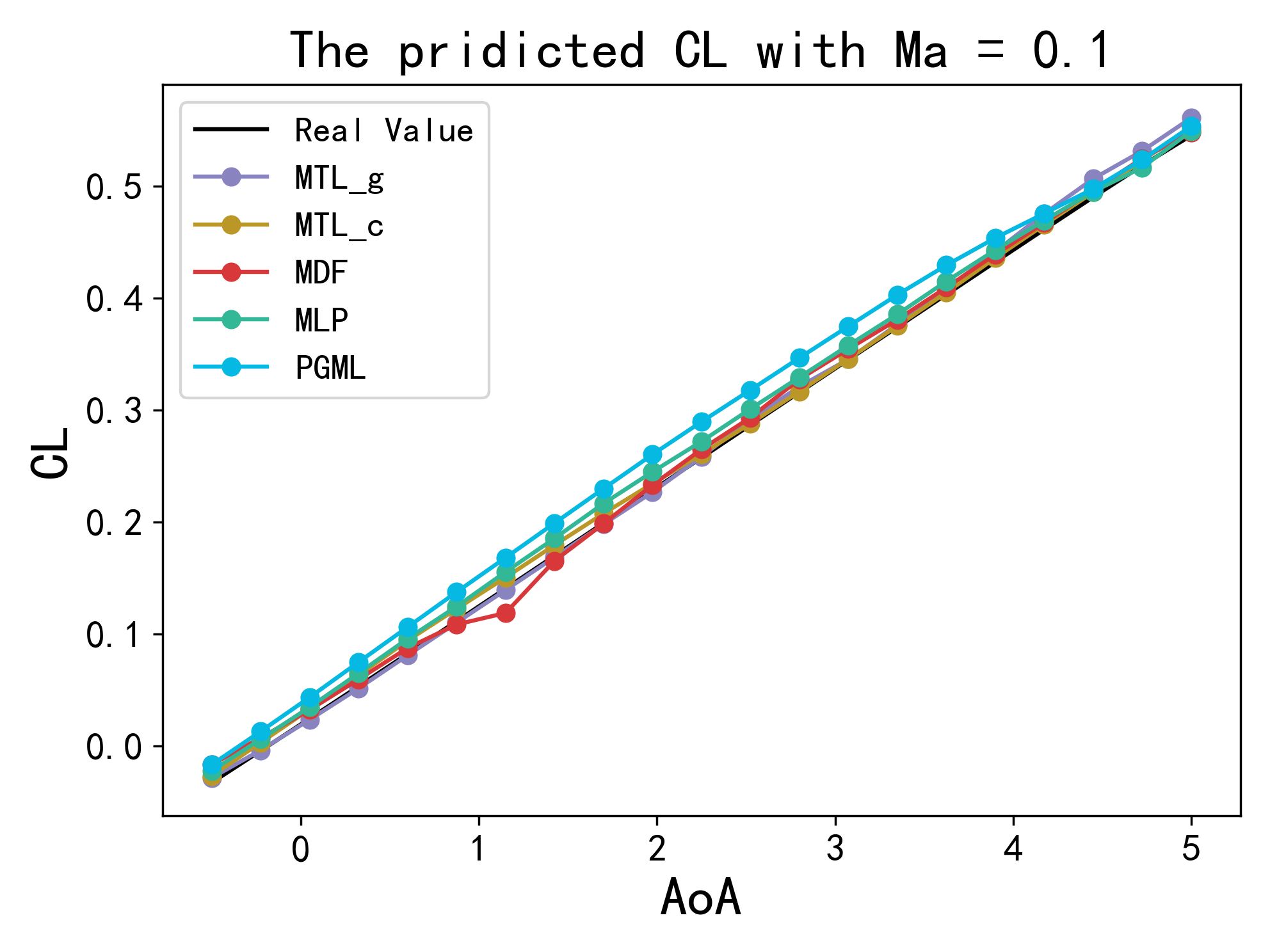}
    }
    \caption{The $C_{L}$ and $C_{D}$ changes with respect to $\alpha$ (i.e., AoA) with Ma = 0.1 in two different round experiments. }
    \label{fig_cd_cl}
\end{figure*}

In the field of Aerodynamics, $C_{D}$ of an aircraft is of a small order of magnitude compared with $C_{L}$, which leads to difficulties in predicting $C_{D}$ \cite{bussy2018effects, timmer2020simple, onel2018drag, marchesse2019drag}. Hence, we pay special attention to the prediction of $C_{D}$ in this paper.

Tab.\ref{tab_MTL_errors} shows the average prediction errors of MTL\_g, MTL\_c, RBF-GAN, PGML, MLP and MDF in 10 rounds. In each column, the figures in bold are the minimum average test errors. we see that, in term of $C_{D}$, both the test MSE and MAE of MDF are smallest among all methods, while both the test MSE and MAE of RBF-GAN are the largest. In term of $C_{L}$, the test MSE and MAE of MLP are the smallest, and both the test MSE and MAE of MDF are in the same order of magnitude as those of MLP. By comparing MDF and MLP, we see that (1) the average MSE of $C_{D}$ predicted by MDF is reduced by $35.37\%$; (2) the average MAE of $C_{D}$ predicted by MDF is reduced by $19.05\%$; (3) the average MSE of $C_{L}$ predicted by MDF is increased by 13.08\%; (4) the average MAE of $C_{L}$ predicted by MDF is increased by 9.11\%. This table indicates that the manifold-features extracted by manifold-based airfoil geometric-feature extraction module can be fused with flight conditions by MDF to reduce the prediction errors of $C_{D}$ while keeping the predicted $C_{L}$ in the same order of magnitude with $C_{L}$ predicted by MLP.

Fig. \ref{fig_loss_distribution} shows the distributions of MSEs and MAEs of $C_{D}$ and $C_{L}$ predicted by MTL\_g, MTL\_c, PGML, MLP and MDF in 10 rounds. Because the predicted errors of RBF-GAN are too large, the distributions of predicted errors by RBF-GAN is omited in this figure. In each subgraph, the blue line at top indicates the maximum value, the blue line in the middle represents the mean value, and the blue line at bottom indicates the minimum value. Besides, the light blue shade indicates the intensity of errors. In subgraph (a), we see that the average MSE of $C_{D}$ predicted by MDF is the smallest among all methods. In subgraph (b), both the minimum and mean MAE of $C_{D}$ predicted by MDF are smaller than those predicted by other methods. In subgraph (c), all the maximum, minimum and mean MSE of $C_{L}$ predicted by MTL\_c, MLP and MDF are similar. In subgraph (d), the mean MAE of $C_{L}$ predicted by MDF is similar to those predicted by MTL\_c and MLP. These subgraph indicates that MDF can reduce the predicted errors of $C_{D}$ while keeping the predicted errors of $C_{L}$ similar to those predicted by MTL\_c and MLP.

Fig.~\ref{fig_loss_2} and Fig.~\ref{fig_loss_9} depict the training loss and validation loss of above methods in two different rounds. Subgraph (a) $\sim$ (f) in Fig.~\ref{fig_loss_2} show the loss variations of MTL\_g, MTL\_c, RBF-GAN, PGML, MLP and MDF in round 2, respectively. Subgraph (a) $\sim$ (f) in Fig.~\ref{fig_loss_9} show the loss variations of MTL\_g, MTL\_c, RBF-GAN, PGML, MLP and MDF in round 9, respectively. We see that the fluctuations of both the training loss and the validation loss of converged MDF (epoch $ \geq 1500$) are smaller than those of other methods. In addition, the fluctuations of the validation loss of MDF are always smaller than those of MTL\_c, PGML and MLP during the entire training process. From comparisons of loss variations, we deduce that it make more sense to fuse manifold-features and flight conditions together by MDF, and this fusion is beneficial to the stability of training process.

Fig.\ref{fig_cd_cl_45_line} shows the $C_{D}$ and $C_{L}$ predicted by MLP, MTL\_c and MDF (three optimal methods in Tab. \ref{tab_MTL_errors}) in the same two rounds. Subgraph (a) and (b) show the predicted $C_{D}$ and $C_{L}$ in round 2, and subgraph (c) and (d) show the predicted $C_{D}$ and $C_{L}$ in round 9. We see that the $C_{D}$ predicted by MDF are closely clustered around the diagonal. On the contrary, the $C_{D}$ predicted by MLP and MTL\_c deviate from the diagonal significantly when $C_{D} > 0.03$. In addition, the $C_{L}$ predicted by three methods are similar.

The variations of predicted $C_{D}$ and $C_{L}$ with $\alpha$ in the same two rounds are shown in Fig. \ref{fig_cd_cl}. Because the predicted errors of RBF-GAN are largest, the $C_{D}$ and $C_{L}$ predicted by RBF-GAN are not shown in this figure. Subgraph (a) and (b) show the variations of $C_{D}$ and $C_{L}$ with $\alpha$ in round 2, and subgraph (c) and (d) show the variations of $C_{D}$ and $C_{L}$ with $\alpha$ in round 9. In subgraph (a) and (c), we see that the $C_{D}$ predicted by MDF are closer to the real value. In the subgraph (b) and (d), the $C_{L}$ predicted by all methods are relatively similar and close to the real value.

In summary, experiments I demonstrated that the manifold-features extracted by the manifold-based airfoil geometric-feature extraction module indeed reflect the geometric shape of airfoils and the manifold-features can be used to generate airfoils based on their geometrical nature. In addition, experiments II demonstrated that the manifold-features and flight conditions can be fused by MDF to reduce the predicted errors of $C_{D}$ while keeping the same predicted accuracy level of $C_{L}$.

\section{Conclusion}
\label{section_conclusion}
The conclusions of our work are as follows:
\begin{enumerate}
\itemsep=0pt
\item  the geometric shape of an airfoil can be approximated by a set of self-intersection-free Bézier curves which are connected end to end to form a segmented smooth Riemannian manifold;
\item  Riemannian metric, as a sort of manifold-feature, can used to re-built smooth and approximated airfoils, and compared with Auto-Encoder, the MSE of re-built airfoil is reduced by 53.66\%;
\item  compared with MLP, the MSE of $C_{D}$ predicted by MDF is reduced by 35.37\% while keeping the same predicted accuracy level of $C_{L}$.
\end{enumerate}

From the results of above experiments, we see that the predicted errors of $C_{D}$ can be significantly reduced by MDF which fuses only one sort of manifold-feature (i.e., Riemannian metric) with flight conditions. In the future, more latent manifold-features (e.g., the curvature, torsion and Riemannian connection) of the segmented smooth manifold build from airfoil curves will be defined and extracted. We firmly believe that 1) multiple manifold-features can further reflect the geometrical nature of airfoils; 2) the fusion of multiple manifold-features with flight conditions can obtain more accurate predictions of airfoil performances.


%

\section{Acknowledgment}
The authors would like to thank Dr. Wenzheng Wang, Research Fellow, from University of Electronic Science and Technology of China for valuable suggestions with our project. In addition, the authors would also like to thank Dr. Yanqing Cheng, Associate Research Fellow, and Dr. Weiqi Qian, Research Fellow, both from China Aerodynamics Research and Development Center for their valuable suggestions with this paper.

\ifCLASSOPTIONcaptionsoff
  \newpage
\fi



\bibliographystyle{IEEEtran}
\bibliography{references}
%
%
%

%

\begin{IEEEbiography}[{\includegraphics[width=1in,height=1.25in,clip,keepaspectratio]{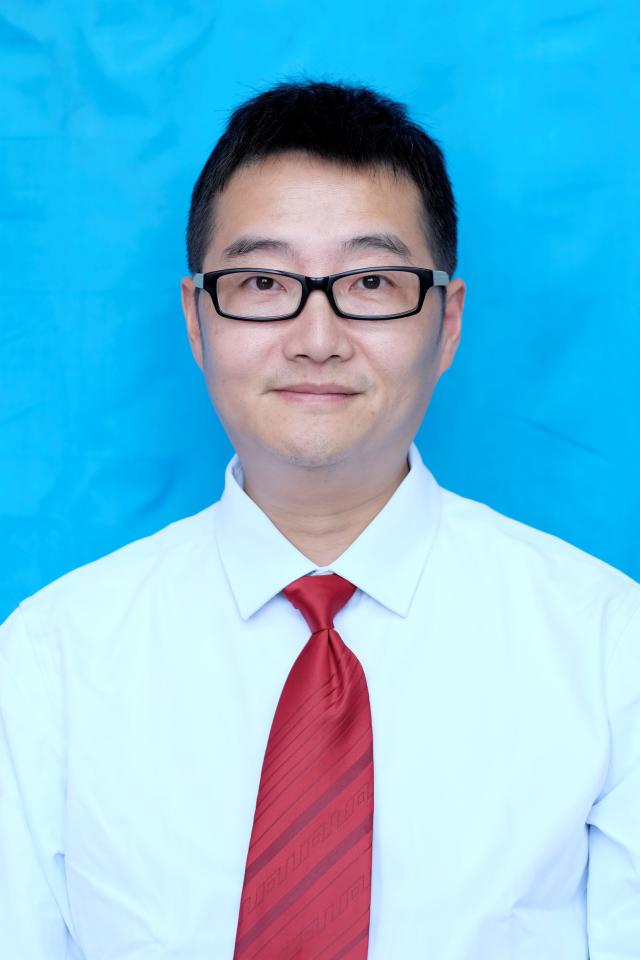}}]{Yu Xiang} received his B.S, M.S. and Ph.D. degrees from the University of Electronic Science and Technology of China (UESTC), Chengdu, Sichuan, China, in 1995, 1998 and 2003, respectively. He joined the UESTC in 2003 and became associate professor in 2006. From 2014-2015, he was a visiting scholar at the University of Melbourne, Australia. His current research interests include computer networks, intelligent transportation systems and deep learning.
\end{IEEEbiography}

\begin{IEEEbiography}[{\includegraphics[width=1in,height=1.25in,clip,keepaspectratio]{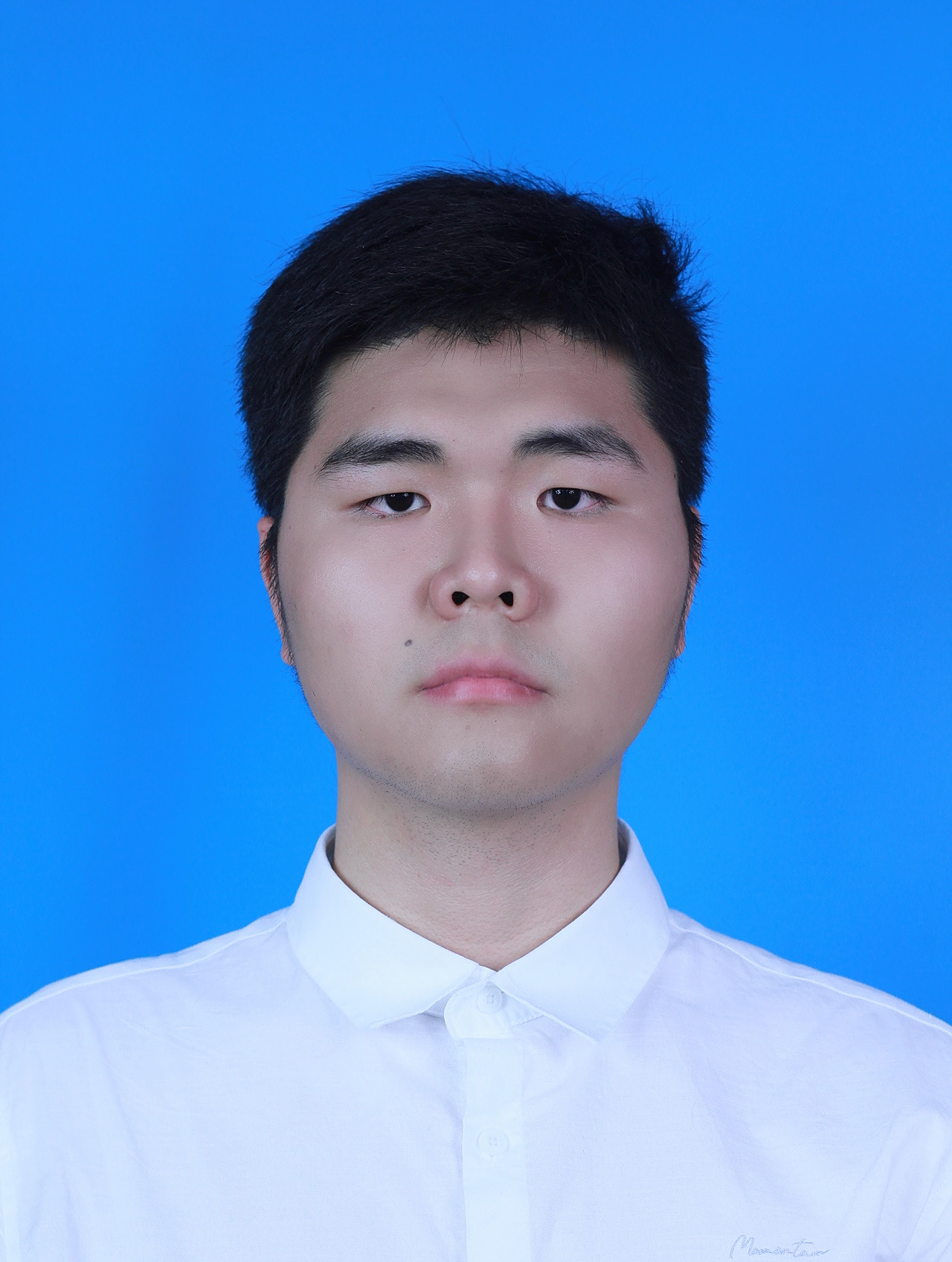}}]{Guangbo Zhang} received his B.S. degree in network engineering from the School of Computer sicience, South-Central Minzu University, Wuhan, Hubei, China, in 2020.He is currently pursuing a M.S. degree in computer science and technology from the School of Computer Science and Engineering, UESTC. His research fields include aerodynamic data modeling and deep learning.
\end{IEEEbiography}

\begin{IEEEbiography}[{\includegraphics[width=1in,height=1.25in,clip,keepaspectratio]{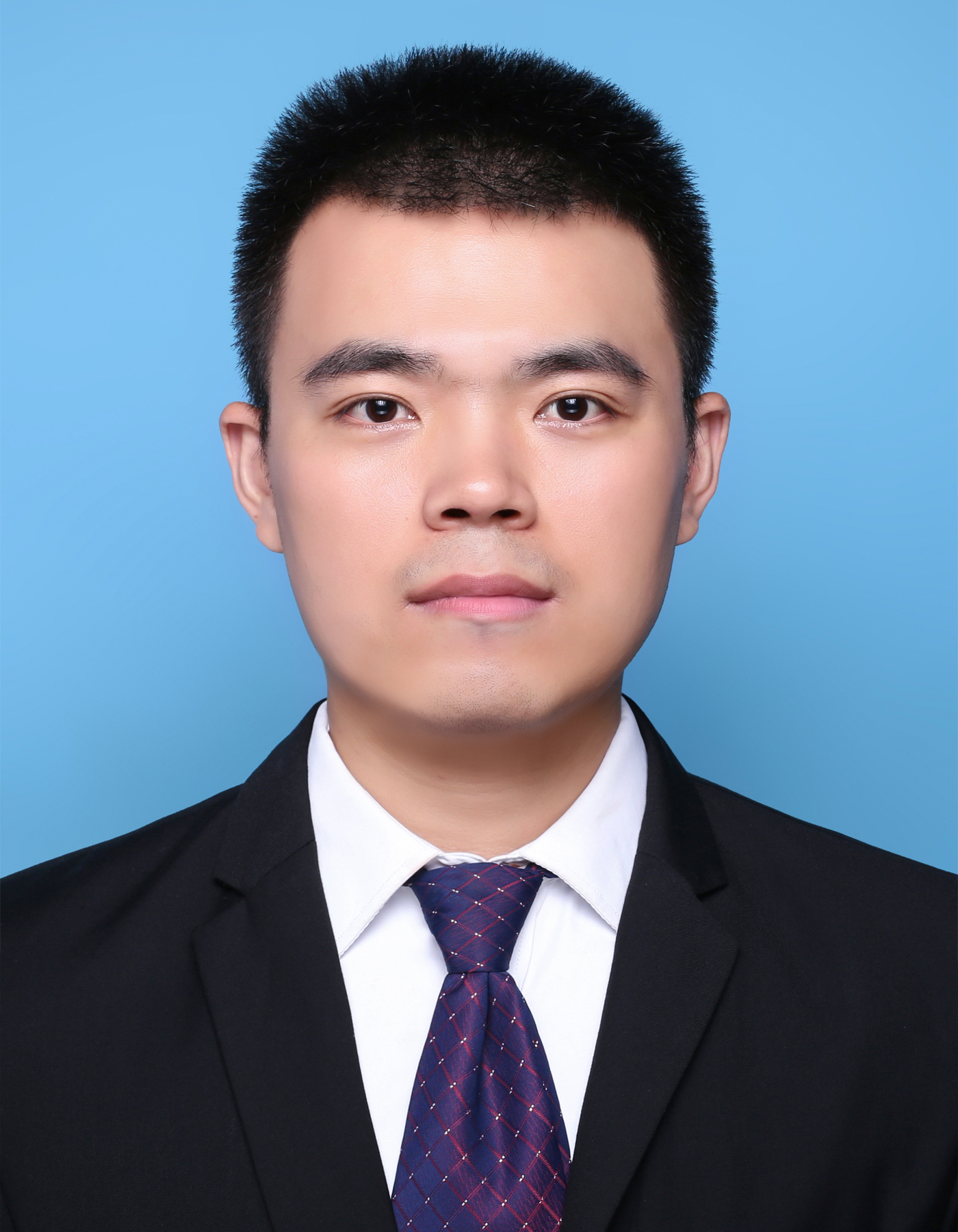}}]{Liwei Hu} received his B.S. degree in software engineering from the School of Software, Hebei Normal University, Shijiazhuang, Hebei, China, in 2014 as well as an M.S. degree in computer technology from the University of Electronic Science and Technology of China (UESTC), Chengdu, Sichuan, China, in 2018. He is currently pursuing a Ph.D. degree in computer science and technology from the School of Computer Science and Engineering, UESTC. His research fields include aerodynamic data modeling, deep learning and pattern recognition.
\end{IEEEbiography}

\begin{IEEEbiography}[{\includegraphics[width=1in,height=1.25in,clip,keepaspectratio]{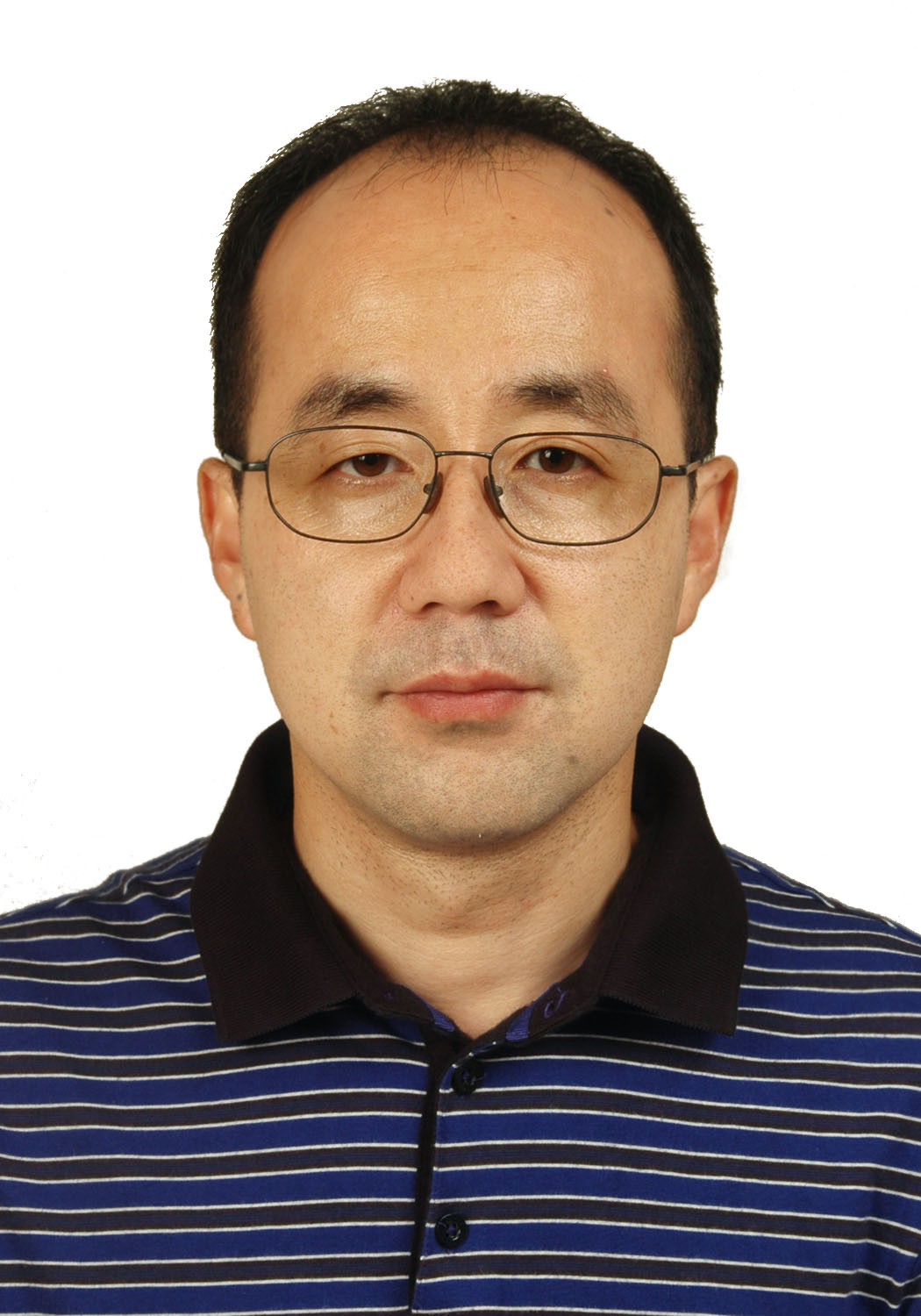}}]{Jun Zhang} received his B.S. and M.S. degrees in electronic engineering from the University of Electronic Science and Technology of China (UESTC) in 1995 and 1998, respectively. From 1998 to 2008, he worked as a senior researcher and engineer in CERNET. He is currently a lecturer at the School of Computer Science and Engineering, UESTC. His current research interests include software-defined networks, machine learning applied in network traffic engineering, and aerodynamics.
\end{IEEEbiography}

\begin{IEEEbiography}[{\includegraphics[width=1in,height=1.25in,clip,keepaspectratio]{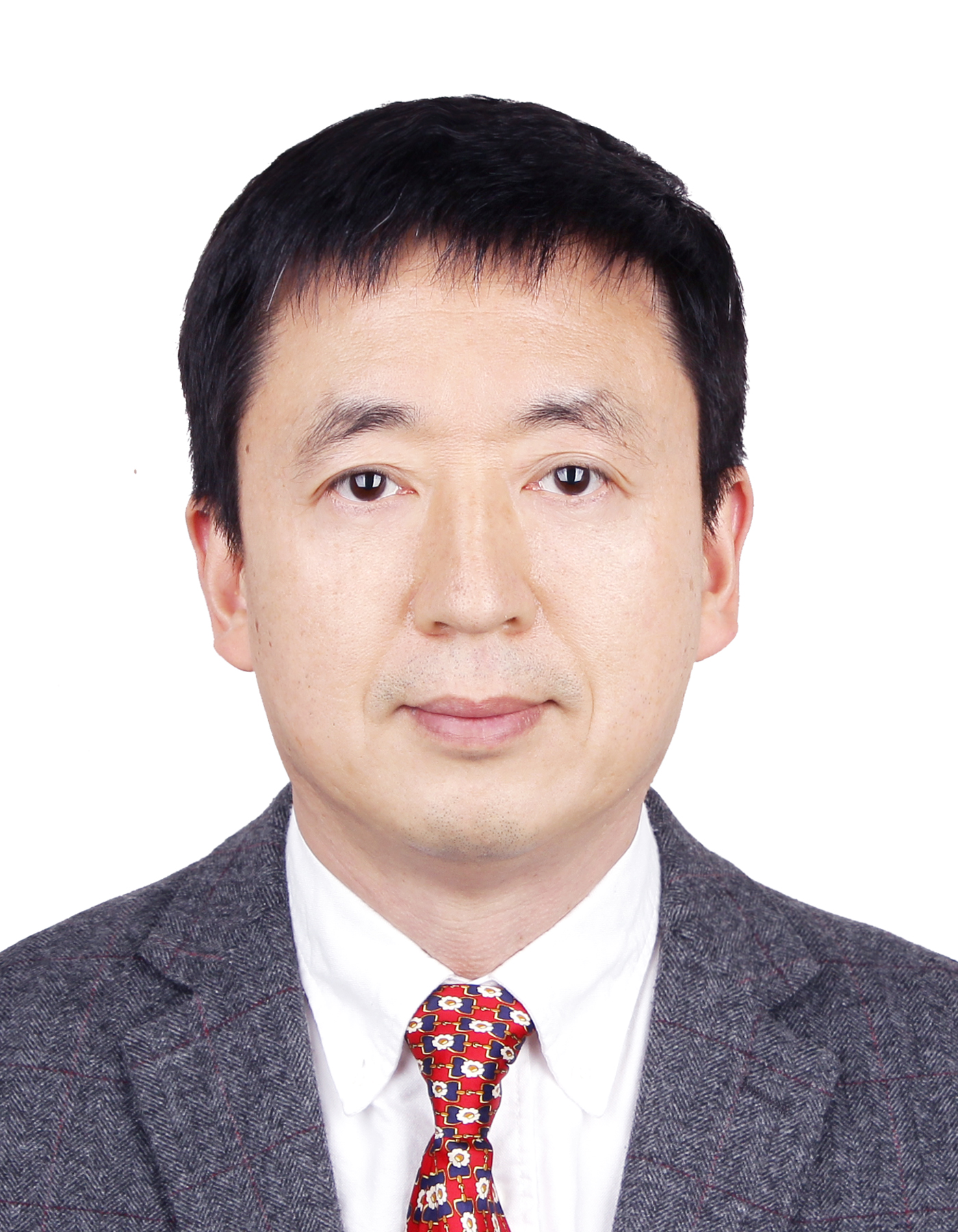}}]{Wenyong Wang} received his B.S. degree in computer science from BeiHang University, Beijing, China, in 1988 and M.S. and Ph.D. degrees from the University of Electronic Science and Technology (UESTC), Chengdu, China, in 1991 and 2011, respectively. He has been a professor in computer science and engineering at UESTC since 2006. Now he is also a special-term professor at Macau University of Science and Technology, a senior member of the Chinese Computer Federation, a member of the expert board of the China Education and Research Network (CERNET) and China Next Generation Internet. His main research interests include next generation Internet, software-defined networks, and software engineering.
\end{IEEEbiography}




\end{document}